%% file: aseoos.tex
\documentclass[letterpaper,12pt]{article}
\usepackage{natbib} \bibpunct{(}{)}{;}{a}{}{,}
\usepackage{fullpage}
\usepackage{amsmath,amssymb,amsthm}
\usepackage{enumerate}
\usepackage{appendix}
\usepackage{graphicx}
\usepackage{color}
\usepackage{hyperref}
\hypersetup{
        colorlinks   = true,
        linkcolor    = blue,
        citecolor    = green
}
\usepackage{authblk}

\input{newcommands}

\begin{document}

\title{Out-of-sample extension of graph adjacency spectral embedding}
\author[1]{Keith Levin}
\author[2,3]{Farbod Roosta-Khorasani}
\author[3,4]{Michael W. Mahoney}
\author[5]{Carey E. Priebe}

\affil[1]{\small Department of Statistics, University of Michigan, USA.}
\affil[2]{\small School of Mathematics and Physics, University of Queensland, Australia.}
\affil[3]{\small International Computer Science Institute, Berkeley, USA.}
\affil[4]{\small Department of Statistics, University of California at Berkeley, USA.}
\affil[5]{\small Department of Applied Mathematics and Statistics, Johns
    Hopkins University}
\maketitle


\begin{abstract}
Many popular dimensionality reduction procedures have out-of-sample extensions, which allow a practitioner to apply a learned embedding to observations not seen in the initial training sample. In this work, we consider the problem of obtaining an out-of-sample extension for the adjacency spectral embedding, a procedure for embedding the vertices of a graph into Euclidean space. We present two different approaches to this problem, one based on a least-squares objective and the other based on a maximum-likelihood formulation. We show that if the graph of interest is drawn according to a certain latent position model called a random dot product graph, then both of these out-of-sample extensions estimate the true latent position of the out-of-sample vertex with the same error rate. Further, we prove a central limit theorem for the least-squares-based extension, showing that the estimate is asymptotically normal about the truth in the large-graph limit.
\end{abstract}

\section{Introduction}
Given a graph $G = (V,E)$ on $n$ vertices
with adjacency matrix $A \in \{0,1\}^{n \times n}$,
the problem of graph embedding
is to map the vertices of $G$ to some $d$-dimensional vector space $\calS$
in such a way that geometry in $\calS$ reflects the topology of $G$.
For example, we may ask that vertices with high conductance in $G$
be assigned to nearby vectors in $\calS$.
This is a special case of the problem of dimensionality reduction,
well-studied in machine learning and related disciplines
\citep{van2009dimensionality}.
When applied to graph data, each vertex in $G$
is described by an $n$-dimensional binary vector,
namely its corresponding column (or row) in adjacency matrix
$A \in \{0,1\}^{n \times n}$, and we wish to associate with each vertex
$v \in V$ a lower-dimensional representation, say $x_v \in \calS$.
The two most commonly-used approaches for graph embeddings are
the graph Laplacian embedding and its variants
\citep{BelNiy2003,CoiLaf2006} and
the adjacency spectral embedding \citep[ASE,][]{SusTanFisPri2012}.
Both of these embedding procedures produce low-dimensional
representations of the vertices in a graph $G$,
and the decision as to which embedding is preferable
is dependent on the downstream task.
Indeed, one can show that neither embedding dominates the other for the purposes of vertex classification; see, for example, Section 4.3 of \citet{TanPri2016}.
In addition, the results in Section 4.3
of \citet{TanPri2016} suggest that ASE performs
better than the Laplacian eigenmaps embedding
for graphs that exhibit a core-periphery structure.
Such structures are ubiquitous in real networks,
such as those arising in social and biological sciences
\citep{JeuETAL2015,LesLanDasMah2009}.

The ASE and Laplacian embedding differ in that the latter has received far more attention, especially with respect to questions of limit objects \citep{HeiAudvon2005} and out-of-sample extensions \citep{BengioETAL2003}. The aim of this paper is to establish theoretical foundations for the latter
of these two problems in the case of the adjacency spectral embedding.

\section{Background and Notation}

In the standard out-of-sample (OOS) extension,
we are presented with training data
$\calD = \{z_1,z_2,\dots,z_n\} \subseteq \calX$,
where $\calX$ is the set of possible observations.
The data $\calD$ give rise to a symmetric matrix
$M = [ K(z_i,z_j) ] \in \R^{n \times n}$,
where $K: \calX \times \calX \rightarrow \Rnonneg$ is a kernel function
that measures similarity between elements of $\calX$, so that
$K(y,z)$ is large if $y,z \in \calX$ are similar, and is small otherwise.
Suppose that we have computed an embedding of the data $\calD$.
Let us denote this embedding by $X \in \R^{n \times d}$,
so that the embedding of $z_i \in \calD$ is given by the $i$-th row of $X$.
Suppose that we are given an additional observation $z \in \calX$,
not necessarily included in $\calD$,
and we wish to embed $z$ under the same scheme as was used to produce $X$.
A na\"ive approach would be to discard the old embedding $ X $, consider the augmented collection
$\calDtilde = \calD \cup \{z\}$ and construct
a new embedding $\Xtilde \in \R^{(n+1)\times d}$.
However, in many applications, it is infeasible to compute this
embedding again from scratch, either because of computational constraints
or because the similarities $\{ K(z_i,z_j) : z_i,z_j \in \calD \}$
may no longer be available after $X$ has been computed.
Thus, the OOS problem is to embed
$z$ using only the available embedding $X$ which was initially learned from $\calD$
and the similarities $\{ K(z_i,z) \}_{i=1}^n$.

As an example, consider the Laplacian eigenmaps embedding
\citep{BelNiy2003,BelNiySin2006}.
Given a graph $G = (V,E)$ with adjacency matrix $A \in \R^{n \times n}$, the
$d$-dimensional \emph{normalized Laplacian} of $G$ is the matrix
$L = D^{-1/2} A D^{-1/2}$, where $ D \in \mathbb{R}^{n \times n} $ is the diagonal degree matrix, i.e., $ d_{ii} = \sum_{j} A_{ij}$ is the degree of the vertex $ i $~\citep{von2007tutorial,vishnoi2013lx}. The $d$-dimensional \emph{normalized Laplacian eigenmaps embedding}
of $G$ is given by the rows of the matrix
$\UL \in \R^{n \times d}$,
whose columns are the $d$ orthonormal eigenvectors
corresponding to the top $d$ eigenvalues of $L$,
excepting the trivial eigenvalue $1$.
We note that some authors \citep[see, for example,][]{Chung1997}
use $I - D^{-1/2} A D^{-1/2}$ to be the normalized graph Laplacian, but since this matrix has the same eigenspace
as our $L$, results concerning the eigenvectors of either
of these matrices are equivalent.
Suppose that a vertex $v$ is added to graph 
$G$, to form graph $\Gtilde$ with adjacency matrix
\begin{align} 
\label{eq:def:Atilde}
\Atilde = \begin{bmatrix} A & \avec \\
\avec^T & 0 \end{bmatrix},
\end{align}
where $\avec \in \{0,1\}^n$.
A na\"ive approach to embedding $\Gtilde$ would be to compute the
top eigenvectors of the graph Laplacian of $\Gtilde$ as before.
However, the OOS extension problem
requires that we only use the information available in
$\UL$ and $\avec$ to compute an embedding of the new vertex $v$.

\citet{BengioETAL2003} presented out-of-sample extensions for
multidimensional scaling \citep[MDS,][]{Torgerson1952,BorGro2005},
spectral clustering \citep{Weiss1999,NgJorWei2002},
Laplacian eigenmaps \citep{BelNiy2003}
and ISOMAP \citep{TendeSLan2000}.
These OOS extensions were based on a least-squares
formulation of the embedding problem, arising from the fact that the
in-sample embeddings are given by functions of the eigenvalues
and eigenfunctions.
\citet{TroPri2008} considered a different OOS extension for MDS.
Rather than following the approach of \citet{BengioETAL2003},
\citet{TroPri2008} cast the MDS OOS extension as a
simple modification of the in-sample MDS optimization problem.

Let $\{(\lambda_{t}, v_t)\}_{t=1}^{n}$ be the eigen-pairs of the matrix $ M $, constructed from some suitably-chosen similarity function, $ K $, defined on pairs of observations in $ \calD \times \calD $. In general, OOS extensions for eigenvector-based embeddings can be derived as in \citet{BengioETAL2003}
as the solution of a least-squares problem
$$ \min_{ f(x) \in \R^d }
\sum_{i=1}^n \left( K(x,x_i) - \frac{1}{n} \sum_{t=1}^d \lambda_{t} f_t(x_i) f_t(x) \right)^2, $$
where $\{x_i\}_{i=1}^n$ are the in-sample observations, and $f_t(x_i) = [v_t]_{i} $ is $ i^{th}$ component of $ v_{t} $.
\citet{BelNiySin2006} presented a slightly different approach
that incorporates regularization in
both the intrinsic geometry of the data distribution
and the geometry of the similarity function $ K $.
Their approach applies to Laplacian eigenmaps
as well as to regularized least squares and SVM.
The authors also introduced a Laplacian SVM,
in which a Laplacian penalty term is added to the 
standard SVM objective function.
\citet{BelNiySin2006} showed that all of these embeddings have
OOS extensions that arise as the solution of a
generalized eigenvalue problem.
We refer the interested reader to \citet{LevJanVan2015}
for a practical application of this OOS extension.
More recent approaches to OOS extension have 
avoided altogether the need to solve a least squares
or eigenvalue problem
by, instead, training a neural net to learn the embedding
directly \citep[see, for example,][]{QuiPetHeu2016,JanSelLyz2017}.

The only existing work to date on the ASE OOS extension
of which we are aware appears in \citet{TanParPri2013}.
The authors considered the OOS extension for ASE applied
to \emph{latent position graphs}
\citep[see, for example][]{HofRafHan2002},
in which each vertex is associated with an element of a vector
space and edge probabilities are given by a suitably-chosen inner product.
The authors introduced a least-squares OOS extension for embeddings
of latent position graphs and proved a theorem, analogous to our
Theorem~\ref{thm:lsrate}, for the error of this
extension about the true latent position.
Theorem~\ref{thm:lsrate} simplifies the proof of the result due to
\citet{TanParPri2013} for the case of random dot product graphs
(see Definition~\ref{def:RDPG} below).

Of crucial importance in assessing OOS extensions,
but largely missing from the existing literature,
is an investigation of how the OOS estimate
compares with the in-sample embedding.
That is, for an out-of-sample observation $z \in \calX$, how well does its OOS embedding $\Xhat_z \in \R^d$, approximate the embedding that would be obtained
by considering the full sample $\calDtilde = \calD \cup \{z\}$?
In this paper, we address this question in the context of
the adjacency spectral embedding. In particular, we show in our main results, Theorems~\ref{thm:lsrate}
and~\ref{thm:mlrate}, that two different approaches to the ASE OOS extension
recover the in-sample embedding at
a rate that is, in a certain sense, optimal
(see the discussion at the end of Section~\ref{sec:theory}).
We conjecture that analogous rate results can be obtained for other
OOS extensions such as those presented in~\citet{BengioETAL2003}.

\subsection{Notation}

We pause briefly to establish notational conventions for this paper.
For a matrix $B \in \R^{n_1 \times n_2}$, we let $\sigma_i(B)$ denote
the $i$-th singular value of $B$, so that
$\sigma_1(B) \ge \sigma_2(B) \ge \dots \ge \sigma_k(B) \ge 0$,
where $k = \min\{n_1,n_2\}$.
For positive integer $n$, we let $[n] = \{1,2,\dots,n\}$.
Throughout this paper, $n$ will index the number of vertices 
in a hollow graph $G$, the observed data,
and we let $c > 0$ denote a positive constant,
not depending on $n$, whose value may change from line to line.
For an event $E$, we let $E^c$ denote its complement.
We will say that event $E_n$, indexed so as to depend on $n$, occurs
\emph{with high probability}, and write $E_n$ w.h.p.\ ,
if for some constant $\epsilon > 0$,
it holds for all suitably large $n$ that
$\Pr[E_n^c] \le n^{-(1+\epsilon)}$.
We say that event $E_n$ occurs \emph{almost surely almost always},
and write $E_n$ a.s.a.a.\ to mean that
with probability $1$, there exists a finite $n_0$ such that
$E_n$ occurs for all $n \ge n_0$, i.e., $\Pr(\liminf_{n} E_{n}) = \Pr( \bigcup_{n=1}^{\infty} \bigcap_{k = n}^{\infty} E_{k}) = 1$.
We note that under these definitions, $E_n$ w.h.p.\
implies $E_n$ a.s.a.a.\ by the Borel-Cantelli Lemma.
In this paper, we will show $\Pr[E^c] \le cn^{-2}$ any time we wish
to show that event $E$ occurs with high probability.
For a function $f : \Znonneg \rightarrow \Rnonneg$
and a sequence of random variables $\{ Z_n \}$, we will write
$Z_n = O( f(n) )$ if there exists a constant $C$
and a number $n_0$ such that
$Z_n \le C f(n)$ for all $n \ge n_0$,
and write $Z_n = O( f(n) )$  a.s.\ if
the event $Z_n \le C f(n)$ occurs a.s.a.a.
For a vector $x \in \R^d$, we use the unadorned norm $\| x \|$ to denote
the Euclidean norm of $x$, and $\| x \|_\infty$ to denote the
supremum norm $\| x \|_\infty = \max_{i \in [d]} |x_i|$.
vector $x \in \R^d$, and to denote the operator norm
For a matrix $M \in \R^{n \times d}$, we use the unadorned norm $\| M \|$
to denote the operator norm
$$ \| M \| = \max_{x \in \R^d : \| x \| = 1} \| M x \| $$
and we use $\| \cdot \|_{2 \rightarrow \infty}$ to denote the matrix
operator norm
$$ \| M \|_{2 \rightarrow \infty} = \max_{x : \| x \| = 1} \| Mx \|_\infty
	= \max_{i \in [n]} \| M_{i,\cdot} \|, $$
which can be proven via the Cauchy-Schwarz inequality \citep{HorJoh2013}.
This latter operator norm will be especially useful for us,
in that a bound on $\| M \|_{2 \rightarrow \infty}$ gives a uniform
bound on the rows of matrix $M$.

\subsection{Roadmap}
The remainder of this paper is structured as follows.
In Section~\ref{sec:OOS}, we present two OOS extensions of the ASE.
In Section~\ref{sec:theory}, we prove convergence of these two OOS
extensions when applied to random dot product graphs.
In Section~\ref{sec:expts}, we explore the empirical performance of the
two extensions presented in Section~\ref{sec:OOS},
and we conclude with a brief discussion in Section~\ref{sec:conclusion}.

\section{Out-of-sample Embedding for ASE} \label{sec:OOS}

Given a graph $G$ encoded by adjacency matrix $A \in \{0,1\}^{n \times n}$,
the adjacency spectral embedding (ASE)
produces a $d$-dimensional embedding of the vertices of $G$,
given by the rows of the $n$-by-$d$ matrix
\begin{equation} \label{eq:def:ase}
	\Xhat = \UA \SA^{1/2},
\end{equation}
where $\UA \in \R^{n \times d}$ is a matrix with orthonormal columns
given by the $d$ eigenvectors corresponding to the top $d$ eigenvalues
of $A$, which we collect in the diagonal matrix $\SA \in \R^{d \times d}$.
We note that in general, one would be better-suited
to consider the matrix $[A^{T}A]^{1/2} $, so that all
eigenvalues are guaranteed to be nonnegative,
but we will see that in the random dot product graph,
the model that is the focus of this paper,
the top $d$ eigenvalues of $A$ are positive
with high probability
(see Lemma~\ref{lem:Psvals} below,
or see
either Lemma 1 in \citet{AthLyzMarPriSusTan2016}
or Observation 2 in \citet{LevAthTanLyzPri2017}.

The random dot product graph \citep[RDPG,][]{YouSch2007} is an
edge-independent random graph model in which the graph structure
arises from the geometry of a set of \emph{latent positions}, i.e., 
vectors associated to the vertices of the graph.
As such, the adjacency spectral embedding is particularly
well-suited to this model.
\begin{definition}\label{def:RDPG}
	\emph{(Random Dot Product Graph)}
	Let $F$ be a distribution on $\R^d$ such that
	$x^T y \in [0,1]$ whenever $x,y \in \supp F$,
	and let $X_1,X_2,\dots,X_n$ be drawn i.i.d. from $F$.
	Collect these $n$ random points
	in the rows of a matrix $X \in \R^{n \times d}$.
	Suppose that (symmetric) adjacency matrix $A \in \{0,1\}^{n \times n}$
	is distributed in such a way that
	\begin{align} \label{eq:rdpg}
		\Pr[ A | X ]=
		\prod_{1 \le i<j \le n}(X_i^TX_j)^{A_{ij}}(1-X_i^TX_j)^{1-A_{ij}}.
	\end{align}
	When this is the case, we write $(A,X) \sim \RDPG(F,n)$.
	If $G$ is the random graph corresponding to adjacency matrix $A$,
	we say that $G$ is a \emph{random dot product graph}
	with {\em latent positions} $X_1,X_2,\dots,X_n$, where $X_i$ is the latent position corresponding to the $i$-th vertex.
\end{definition}
A number of results exist showing that the adjacency spectral embedding
yields consistent estimates of the latent positions
in a random dot product graph \citep{SusTanFisPri2012,TanSusPri2013}
and recovers community structure in the
stochastic block model \citep{LyzSusTanAthPri2014}.
We note an inherent nonidentifiability in the random dot product graph,
arising from the fact that for any orthogonal matrix $W \in \R^{d \times d}$,
the latent positions $X \in \R^{n \times d}$ and $XW \in \R^{d \times d}$
give rise to the same distribution over graphs,
since $XX^T = (XW)(XW)^T = \E[ A \mid X ]$.
Owing to this nonidentifiability, we can only hope to recover the latent
positions in $X$ up to some orthogonal rotation.

Suppose that, given adjacency matrix $A$, we compute embedding
$$ \Xhat = [ \Xhat_1 \Xhat_2 \dots \Xhat_n ]^T, $$
where $\Xhat_i \in \R^d$ denotes the embedding of the $i$-th vertex.
Now suppose we add a vertex $v$ with latent position
$\wtrue \in \R^d$ to the original graph $G$,
obtaining an augmented graph $\Gtilde = ( [n] \cup \{v\}, E \cup E_v )$,
where $E_v$ denotes the set of edges between $v$ and the vertices of $G$.
One would like to embed vertex $v$ according to the same distribution
as the original $n$ vertices and obtain an estimate of $\wtrue$.
Let the binary vector $\avec \in \{0,1\}^n$
encode the edges $E_v$ incident upon vertex $v$,
with entries $a_i = (\avec)_i \sim \Bernoulli( X_i^T \wtrue )$.
The augmented graph $\Gtilde$ then has the adjacency matrix as in~\eqref{eq:def:Atilde}. As discussed earlier,
the natural approach to embedding vertex $v$ is to simply re-embed the
whole matrix $\Gtilde$ by computing the ASE of $\Atilde$.
Suppose that we wish to avoid such a computation,
for example due to resource constraints.
The problem then becomes one of embedding the new vertex $v$
based solely on the information present in $\Xhat$ and $\avec$.
Two natural approaches to such an OOS extension
suggest themselves.

\subsection{Linear Least Squares OOS}
A natural approach to OOS embedding, pursued by, for example,
\citet{BengioETAL2003}, is to embed vertex $v$
as the least-squares solution to $\Xhat w = \avec$.
That is, we embed the vertex $v$ as the vector $\whatls$ solving
\begin{equation} \label{eq:def:llshat}
	\min_{w \in \R^d} \sum_{i=1}^n \left( a_i - \Xhat_i^T w \right)^2,
\end{equation}
where $a_i$ denotes the $i$-th component of the binary vector
$\avec$ encoding the edges between $v$ and the original $n$ vertices. We will denote the solution to
the least-squares optimization in
Equation~\eqref{eq:def:llshat}
by $\whatls$, and term this the
{\em linear least squares out-of-sample} (LLS OOS) embedding.

\subsection{Maximum Likelihood OOS}
A more principled approach to OOS extension,
but perhaps more involved computationally,
is to consider the following maximum-likelihood formulation.
The entries of the vector $\avec$ are distributed
independently as
$a_i \sim \Bernoulli( X_i^T \wtrue )$, where $\wtrue$ denotes the
true latent position of OOS vertex $v$.
Since we do not have access to the latent positions
$\{ X_i \}_{i=1}^n$, we use instead their
estimates $\{ \Xhat_i \}_{i=1}^n$.
This yields the following objective:
\begin{equation} \label{eq:def:eigenlikhatunbound}
	\max_{w \in \R^d} \sum_{i=1}^n a_i \log \Xhat_i^T w 
	+ (1-a_i)\log \left( 1 - \Xhat_i^T w \right).
\end{equation}
Unfortunately, this optimization problem may fail to
achieve its optimum inside the support of $F$.
Indeed, it may not even have a finite solution.
Thus, we will instead settle for solving the following constrained modification
of Equation~\eqref{eq:def:eigenlikhatunbound},
\begin{equation} \label{eq:def:eigenlikhat}
	\max_{w \in \calThat_\epsilon}
	\sum_{i=1}^n a_i \log \Xhat_i^T w 
	+ (1-a_i)\log \left( 1 - \Xhat_i^T w \right),
\end{equation}
where $\calThat_\epsilon =
\{ w \in \R^d : \epsilon \le \Xhat_i^T w \le 1-\epsilon, i \in [n] \}$,
and $\epsilon > 0$ is a small constant.
We note that this is based only on the edges incident on the OOS vertex rather than on the
full data $\Atilde$,
and uses the spectral estimates $\{\Xhat_i \}_{i=1}^n$
rather than the true latent positions
$\{ X_i \}_{i=1}^n$.
Despite both of these facts,
we will term the extension given by
Equation~\eqref{eq:def:eigenlikhat} as the
{\em maximum-likelihood out-of-sample} (ML OOS) extension,
and we will let $\whatml$ denote its solution.

\section{Main Results} \label{sec:theory}

Our main results show that both the linear least-squares and maximum-likelihood
OOS extensions in Equations \eqref{eq:def:llshat}
and \eqref{eq:def:eigenlikhat}
recover the true latent position $\wtrue$ of $v$.
Further, both OOS extensions converge to $\wtrue$
at the same asymptotic rate (i.e., up to a constant)
as we would have obtained, had
we computed the ASE of $\Atilde$ in~\eqref{eq:def:Atilde} directly.
This rate is given by Lemma 2.5 from \citet{LyzSusTanAthPri2014},
which we state here in a slightly adapted form.
The lemma states, in essence, that the ASE recovers the latent positions
with error of order $n^{-1/2} \log n$, uniformly over the $n$ vertices.
We remind the reader that $\| M \|_{2 \rightarrow \infty}$ denotes
the $2$-to-$\infty$ operator norm,
$ \| M \|_{2 \rightarrow \infty} = \max_{x : \|x\| =1} \| Mx \|_{\infty}$.

\begin{lemma}[Adapted from \citet{LyzSusTanAthPri2014}, Lemma 2.5]
	\label{lem:perfect}
	Let $X = [X_1,X_2,\dots,X_n]^T \in \R^{n \times d}$
	be the matrix of latent positions of an RDPG,
	and let $\Xhat \in \R^{n \times d}$ denote the
	matrix of estimated latent positions yielded by ASE
	as in \eqref{eq:def:ase}.
	Then with probability at least $1-cn^{-2}$,
	there exists orthogonal matrix $W \in \R^{d \times d}$ such that
	$$ \| \Xhat - XW \|_{2 \rightarrow \infty} \le \frac{ c \log n }{ n^{1/2} }. $$
	That is, it holds with high probability that
	for all $i \in [n]$,
	$$ \| \Xhat_i - W^T X_i \| \le \frac{ c \log n }{ n^{1/2} }. $$
\end{lemma}

In what follows, we let
$A \in \{0,1\}^{n \times n}$ denote the random adjacency matrix
of an RDPG $G$, and let $X_1,X_2,\dots,X_n \in \R^d$ denote its latent
positions, collected in matrix
$X = [X_1,X_2,\dots,X_n]^T \in \R^{n \times d}$.
That is, $(A,X) \sim \RDPG(F,n)$.
We use
$\Xhat = [\Xhat_1,\Xhat_2,\dots,\Xhat_n]^T \in \R^{n \times d}$
to denote
the matrix whose rows are the estimated latent positions,
obtained via ASE as in \eqref{eq:def:ase}.
We let $\wtrue$ denote the true latent position
of the OOS vertex $v$.

\begin{theorem} \label{thm:lsrate}
	With notation as above, let $\whatls$ denote the least-squares estimate
	of $\wtrue$, i.e., the solution to \eqref{eq:def:llshat}.
	Then there exists an orthogonal matrix $W \in \R^{d \times d}$
	such that with high probability,
	$$ \| W \whatls - \wtrue \| \le c n^{-1/2} \log n . $$
\end{theorem}
\begin{proof}
	The proof of this result relies upon Lemma~\ref{lem:perfect},
	along with Lemmas~\ref{lem:ls:hat} and~\ref{lem:ls:true},
	both of which are proven in the appendix.
	Lemma~\ref{lem:ls:hat} uses a classic result for solutions of perturbed
	linear systems to establish that with high probability,
	$\| W \whatls - \wls \| \le c n^{-1/2} \log n$, where $W \in \R^{d \times d}$
	is the orthogonal matrix guaranteed by Lemma~\ref{lem:perfect} and $ \wls  $ is the LS estimate based on the true latent positions $\{ X_i \}$ rather than on the estimates $\{ \Xhat_i \}$.
	Lemma~\ref{lem:ls:true} applies a basic Hoeffding inequality to show that
	with high probability, $\| \wls - \wtrue \| \le c n^{-1/2} \log n$,
	where again $W \in \R^{d \times d}$ is the orthogonal matrix
	in Lemma~\ref{lem:perfect}.
	A triangle inequality applied to $\| W\whatls - \wtrue\|$
	combined with a union bound over the events
	in Lemmas~\ref{lem:ls:hat} and~\ref{lem:ls:true} yields the result.
\end{proof}

As mentioned in Section~\ref{sec:OOS},
we would like to consider
a maximum-likelihood OOS extension
based on the likelihood
$ \ellhat(w) = \sum_{i=1}^n a_i \log \Xhat_i^T w
+(1-a_i) \log (1-\Xhat_i^T w).$
Toward this end,
we would ideally like to use the solution to the optimization problem
\begin{equation*}
	\arg \max_{w \in \R^d} \ellhat(w),
\end{equation*}
but to ensure a sensible solution, we instead consider
\begin{equation} \label{eq:whatmlconstr}
	\whatml = \arg \max_{w \in \calThat_{\epsilon}}\ellhat(w),
\end{equation}
where we remind the reader that
$ \calThat_\epsilon = 
\{ w \in \R^d : \epsilon \le \Xhat_i^T w \le 1-\epsilon, i=1,2,\dots,n \}$.
Theorem~\ref{thm:mlrate} shows that $\whatml$
recovers the true latent position of the OOS
vertex, up to rotation, with error decaying at the
same rate as that obtained in Theorem~\ref{thm:lsrate}
for the LS OOS extension.

\begin{theorem} \label{thm:mlrate}
	With notation as above,
	let $\whatml$ be the estimate defined in
	Equation~\eqref{eq:whatmlconstr},
	and let $\epsilon > 0$ be such that
	$x,y \in \supp F$ implies $\epsilon < x^T y < 1-\epsilon$.
	Denote the true latent position of the OOS
	vertex $v$ by $\wtrue \in \supp F$.
	Then for all $n$ suitably large,
	there exists an orthogonal matrix
	$W \in \R^{d \times d}$
	such that with high probability,
	$$ \| W \whatml - \wtrue \| \le cn^{-1/2} \log n, $$
	and this matrix $W$ is the same one guaranteed by Lemma~\ref{lem:perfect}.
\end{theorem}
\begin{proof}
	Lemma~\ref{lem:cvxopt} applies
	a standard argument from convex optimization,
	alongside the definition of $\calThat_\epsilon$,
	to show that for suitably large $n$,
	$$ \| W \whatml - \wtrue \|
	\le \frac{ c \| \nabla \ellhat( W^T \wtrue ) \| }{ n } \text{ w.h.p.} $$
	Lemma~\ref{lem:nabla} then applies a triangle inequality to show that
	$$ \| \nabla \ellhat( W^T \wtrue ) \| \le  c \sqrt{n} \log n
		\text{ w.h.p.} $$
	Details are given in the appendix.
\end{proof}

\begin{remark}
	\textnormal{
		Given our in-sample embedding $\Xhat$
		and the vector of edge indicators $\avec$,
		we can think of the OOS extension as an estimate of $\wtrue$,
		the latent position of the OOS vertex $v$.
		Lemma~\ref{lem:perfect} implies that if we took the na\"ive approach
		of applying ASE to the adjacency matrix $\Atilde$
		in~\eqref{eq:def:Atilde},
		our estimate would have error of order at most $O(n^{-1/2} \log n)$.
		Theorems~\ref{thm:lsrate} and~\ref{thm:mlrate} imply that the
		OOS estimate obtains the
		\emph{same asymptotic estimation error},
		without recomputing the embedding of $\Atilde$}.
\end{remark}

In addition to the bounds in
Theorems~\ref{thm:lsrate} and~\ref{thm:mlrate},
we can show that the least-squares OOS extension
satisfies a stronger property, namely the following central limit theorem.

\begin{theorem} \label{thm:clt}
Let $(A,X) \sim \RDPG(F,n)$ be a $d$-dimensional RDPG.
Let $\wtrue \in \supp F$ and $\whatls \in \R^d$ be, respectively, the latent position and the least-squares embedding from~\eqref{eq:def:llshat} of an OOS vertex $v$.
There exists a sequence of orthogonal $d \times d$ matrices
$\{V_n\}_{n=1}^\infty$ such that
$$ \sqrt{n}( V_n^T \whatls - \wtrue ) \inlaw \calN(0,\Sigma_{\wtrue}), $$
where $\Sigma_{\wtrue} \in \R^{d \times d}$ is given by
\begin{equation} \label{eq:def:Sigma}
\Sigma_{\wtrue}
	= \Delta^{-1} \E\left[X_1^T \wtrue(1-X_1^T\wtrue)X_1 X_1^T \right]
	\Delta^{-1}, \end{equation}
and $\Delta = \E X_1 X_1^T$.
\end{theorem}
\begin{proof}
Details are given in the appendix.
\end{proof}

If the OOS vertex is distributed according to $F$,
we have the following corollary by
integrating $\wtrue$ with respect to $F$.
\begin{corollary} \label{cor:clt}
Let $(A,X) \sim \RDPG(F,n)$ be a $d$-dimensional RDPG,
and let $\wtrue$ be distributed according to $F$, independent of $(A,X)$.
Then there exists a sequence of orthogonal $d \times d$ matrices
        $\{V_n\}_{n=1}^\infty$ such that
$$ \sqrt{n}( V_n^T \whatls - \wtrue ) \inlaw
	\int \calN(0, \Sigma_{w}) dF(w), $$
where $\Sigma_w$ is defined as in Equation~\eqref{eq:def:Sigma} above.
\end{corollary}

We conjecture that a CLT analogous to Theorem~\ref{thm:clt}
holds for the ML OOS extension.

\begin{figure*}
  \centering
  \includegraphics[width=\textwidth]{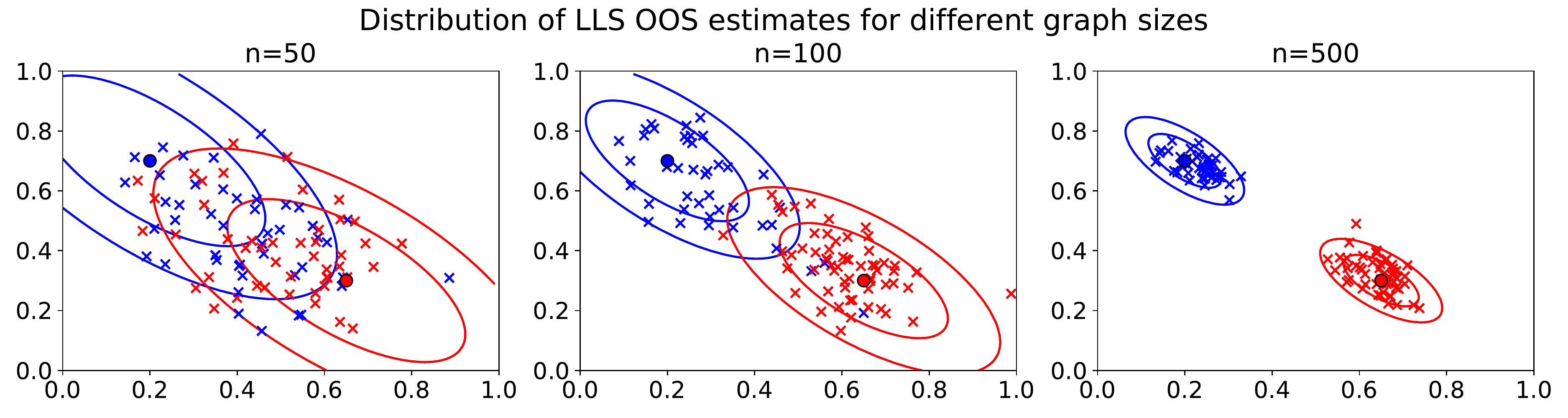}
  \vspace{-5mm}
  \caption{Empirical distribution of the LLS OOS estimate for 100
	independent trials for number of vertices
	$n=50$ (left), $n=100$ (middle) and $n=500$ (right).
	Each plot shows the positions of 100 independent OOS
	embeddings, indicated by crosses,
	and colored according to cluster membership.
	Contours indicate two generalized standard deviations of the
	multivariate normal (i.e., 68\% and 95\% of the probability mass)
	about the true latent positions, which are indicated by solid circles.
	We note that even with merely 100 vertices, the normal
	approximation is already quite reasonable.}
  \label{fig:lls_nplot}
\end{figure*}

\section{Experiments} \label{sec:expts}

In this section, we briefly explore our results through simulations.
We leave a more thorough experimental examination of our results,
particularly as they apply to real-world data,
for future work.
We first give a brief exploration of how quickly the asymptotic
distribution in Theorem~\ref{thm:clt} becomes a good approximation.
Toward this end, let us consider a simple mixture of point masses,
$ F = F_{\lambda,x_1,x_2} = \lambda \delta_{x_1} + (1-\lambda) \delta_{x_2}$,
where $x_1,x_2 \in \R^2$ and $\lambda \in (0,1)$.
This corresponds to a two-block stochastic block model
\citep{Holland1983},
in which the block probability matrix is given by
$$ \begin{bmatrix} x_1^T x_1 & x_1^T x_2 \\
		x_1^Tx_2 & x_2^Tx_2 \end{bmatrix}. $$
Corollary~\ref{cor:clt} implies that if all latent positions
(including the OOS vertex) are drawn according to $F$,
then the OOS estimate should be distributed as a mixture
of normals centered at $x_1$ and $x_2$, with respective mixing coefficients
$\lambda$ and $1-\lambda$.

To assess how well the asymptotic distribution predicted by
Theorem~\ref{thm:clt} and Corollary~\ref{cor:clt} holds,
we generate RDPGs with latent positions
drawn i.i.d. from distribution $F = F_{\lambda,x_1,x_2}$ defined above, with
$$ \lambda = 0.4,~x_1 = (0.2,0.7)^T, \text{ and } x_2 = (0.65, 0.3)^T. $$
For each trial, we draw $n+1$ independent latent positions from
$F$, and generate a binary adjacency matrix from these latent positions.
We let the $(n+1)$-th vertex be the OOS vertex.
Retaining the subgraph induced by the first $n$ vertices, we obtain
an estimate $\Xhat \in \R^{n \times 2}$ via ASE,
from which we obtain an estimate for the OOS vertex
via the LS OOS extension as defined in~\eqref{eq:def:llshat}.
We remind the reader that for each RDPG draw,
we initially recover the latent positions only up to a rotation.
Thus, for each trial, we compute a Procrustes alignment \citep{GowDij2004}
of the in-sample estimates $\Xhat$ to their true latent positions.
This yields a rotation matrix $R$, which we apply to the OOS estimate.
Thus, the OOS estimates are sensibly comparable across trials.
Figure~\ref{fig:lls_nplot}
shows the empirical distribution of the OOS embeddings
of 100 independent RDPG draws, for $n=50$ (left),
$n=100$ (center) and $n=500$ (right) in-sample vertices.
Each cross is the location of the OOS estimate for a single
draw from the RDPG with latent position distribution $F$,
colored according to true latent position.
OOS estimates with true latent position
$x_1$ are plotted as blue crosses, while OOS estimates
with true latent position $x_2$ are plotted as red crosses.
The true latent positions $x_1$ and $x_2$ are plotted as solid circles,
colored accordingly.
The plot includes contours for the two normals centered at $x_1$ and $x_2$
predicted by Theorem~\ref{thm:clt} and Corollary~\ref{cor:clt},
with the ellipses indicating the isoclines corresponding
to one and two (generalized) standard deviations.

\begin{figure*}
  \centering
  \includegraphics[width=\textwidth]{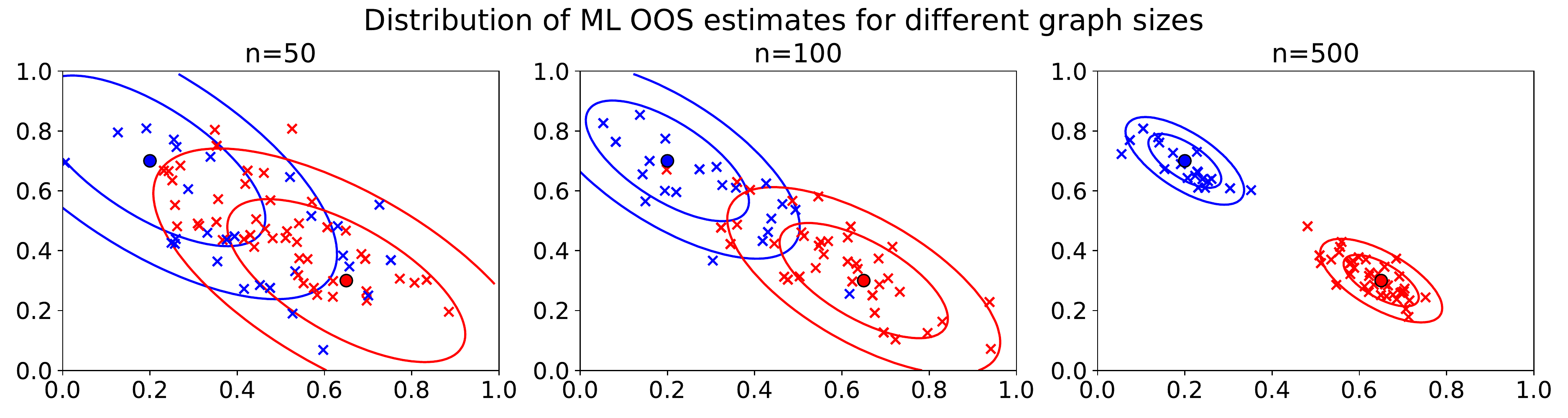}
  \vspace{-5mm}
  \caption{Empirical distribution of the ML OOS estimate for 100
        independent trials for number of vertices
        $n=50$ (left), $n=100$ (middle) and $n=500$ (right).
        Each plot shows the positions of 100 independent OOS
        embeddings, indicated by crosses,
        and colored according to cluster membership.
        Contours indicate two generalized standard deviations of the
        multivariate normal (i.e., 68\% and 95\% of the probability mass)
        about the true latent positions, which are indicated by solid circles.
	Once again, even with merely 100 vertices, the normal
        approximation is already quite reasonable, supporting our conjecture
	that the ML OOS estimates also distributed as a mixture of normals
	according to the latent position distribution $F$.}
  \label{fig:ml_nplot}
\end{figure*}

Examining Figure~\ref{fig:lls_nplot}, we see that even with only 100 vertices,
the mixture of normal distributions predicted by Theorem~\ref{thm:clt}
holds quite well, with the exception of a few gross outliers
from the blue cluster. With $n=500$ vertices, the approximation is
particularly good. Indeed, the $n=500$ case appears to be slightly
under-dispersed, possibly due to the Procrustes alignment.
It is natural to wonder whether a similarly good fit is exhibited by the
ML-based OOS extension.
We conjectured at the end of Section~\ref{sec:theory} that a CLT
similar to that in Theorem~\ref{thm:clt} would also hold
for the ML-based OOS extension as defined in
Equation~\eqref{eq:whatmlconstr}.
Figure~\ref{fig:ml_nplot} shows the empirical distribution of 100 independent
OOS estimates, under the same experimental setup as Figure~\ref{fig:lls_nplot},
but using the ML OOS extension
rather than the linear least-squares extension.
The plot supports our conjecture that the ML-based OOS estimates are
also approximately normally distributed about the true latent positions.

Figure~\ref{fig:lls_nplot} suggests that we may be confident in applying the
large-sample approximation suggested by Theorem~\ref{thm:clt}
and Corollary~\ref{cor:clt}.
Applying this approximation allows us to investigate the trade-offs
between computational cost and classification accuracy,
to which we now turn our attention.
The mixture distribution $F_{\lambda,x_1,x_2}$ above suggests a
task in which, given an adjacency matrix $A$, we wish to classify
the vertices according to which of two clusters or communities they belong.
That is, we will view two vertices as belonging to the same community if
their latent positions are the same
\citep[i.e., the latent positions specify an SBM,]{Holland1983}.
More generally, one may view the task of recovering vertex block memberships
in a stochastic block model as a clustering problem.
\citet{LyzSusTanAthPri2014} showed that applying ASE to such a graph,
followed by $k$-means clustering of the estimated latent positions,
correctly recovers community memberships of all the vertices
(i.e., correctly assigns all vertices to their true latent positions)
with high probability.

For concreteness, let us consider a still simpler mixture model,
$F = F_{\lambda,p,q} = \lambda \delta_p + (1-\lambda) \delta_q$,
where $0 < p < q < 1$,
and draw an RDPG $(\Atilde,X) \sim \RDPG(F,n+m)$,
taking the first $n$ vertices to be in-sample,
with induced adjacency matrix $A \in \R^{n \times n}$.
That is, we draw the full matrix
$$ \Atilde = \begin{bmatrix} A & B \\
                B^T & C \end{bmatrix}, $$
where $C \in \R^{m \times m}$
is the adjacency matrix of the subgraph induced by the $m$ OOS vertices
and $B \in \R^{n \times m}$ encodes the edges between the in-sample
vertices and the OOS vertices.
The latent positions $p$ and $q$ encode a community structure in
the graph $\Atilde$, and, as alluded to above,
a common task in network statistics
is to recover this community structure.
Let $\wtrue^{(1)}, \wtrue^{(2)}, \dots, \wtrue^{(m)} \in \{p,q\}$ denote the
true latent positions of the $m$ OOS vertices,
with respective least-squares OOS estimates
$\whatls^{(1)}, \whatls^{(2)}, \dots, \whatls^{(m)}$,
each obtained from the in-sample ASE $\Xhat \in \R^n$ of $A$.
We note that one could devise a different OOS embedding
procedure that makes use of the subgraph $C$ induced by these $m$
OOS vertices, but we leave the development of such a method
to future work.
Corollary~\ref{cor:clt} implies that each $\whatls^{(t)}$ for $t \in [m]$
is marginally (approximately) distributed as 
$$ \whatls^{(t)} \sim \lambda \calN(p,(n+1)^{-1}\sigma^2_p)
+ (1-\lambda)\calN(q,(n+1)^{-1}\sigma^2_q), $$
where
\begin{equation*} \begin{aligned}
\sigma^2_p &= \Delta^{-2} \left(\lambda p^2(1-p^2)p^2
	+ (1-\lambda)pq(1-pq)q^2 \right), \\
\sigma^2_q &= \Delta^{-2} \left( \lambda pq(1-pq)p^2
+ (1-\lambda)q^2(1-q^2)q^2 \right), \\
\text{ and } \Delta &= \lambda p^2 + (1-\lambda)q^2.
\end{aligned} \end{equation*}
Classifying the $t$-th OOS vertex based on $\whatls^{(t)}$
via likelihood ratio thus has (approximate) probability of error
\begin{equation*}
\eta_{n,p,q}
= \lambda(1 - \Phi\left( \frac{\sqrt{n+1}(x_{n+1,p,q} - p) }{ \sigma_p } \right)
+ (1-\lambda)\Phi\left( \frac{\sqrt{n+1}(x_{n+1,p,q} - q)}{ \sigma_q  } \right),
\end{equation*}
where $\Phi$ denotes the cdf of the standard normal and
$x_{n,p,q}$ is the value of $x$ solving
\begin{equation*}
\lambda \sigma_p^{-1} \exp\{ n(x-p)^2/(2\sigma^2_p) \}
= (1-\lambda) \sigma_q^{-1} \exp\{ n(x-q)^2/(2\sigma^2_q) \} ,
\end{equation*}
and hence our overall error rate
when classifying the $m$ OOS vertices will grow as $m \eta_{n+1,p,q}$.

As discussed previously,
the OOS extension allows us to avoid
the expense of computing the ASE of the full matrix
$$ \Atilde = \begin{bmatrix} A & B \\
		B^T & C \end{bmatrix}. $$
The LLS OOS extension is computationally inexpensive,
requiring only the computation of the
matrix-vector product $\SA^{-1/2} \UA^T \avec$,
with a time complexity $O( d^2 n )$ (assuming one does not precompute
the product $\SA^{-1/2} \UA^T$).
The eigenvalue computation required for embedding
$\Atilde$ is far more expensive
than the LLS OOS extension.
Nonetheless, if one were intent on reducing the OOS classification error
$\eta_{n+1,p,q}$, one might consider paying the computational
expense of embedding $\Atilde$ to obtain estimates
$\wtilde^{(1)}, \wtilde^{(2)}, \dots, \wtilde^{(m)}$
of the $m$ OOS vertices.
That is, we obtain estimates for the $m$ OOS vertices
by making them in-sample vertices, at the expense of solving
an eigenproblem on the $(m+n)$-by-$(m+n)$ adjacency matrix.
Of course, the entire motivation of our approach is that the in-sample
matrix $A$ may not be available.
Nonetheless, a comparison against this baseline,
in which all data is used to compute our embeddings, is instructive.

Theorem 1 in \citet{AthLyzMarPriSusTan2016} implies that the
$\wtilde^{(t)}$ estimates based on embedding the full matrix $\Atilde$ are
(approximately) marginally distributed as
$$ \wtilde^{(t)} \sim \lambda \calN(p,(n+m)^{-1}\sigma^2_p)
+ (1-\lambda)\calN(q,(n+m)^{-1}\sigma^2_q), $$
with classification error
\begin{equation*}
\eta_{n+m,p,q}
= \lambda \Phi\left( \frac{ p - x_{n+m,p,q} }{ \sigma_p } \right)
+ (1-\lambda)\Phi\left( \frac{ x_{n+m,p,q} - q }{ \sigma_q  } \right),
\end{equation*}
where $x_{n+m,p,q}$ is the value of $x$ solving
\begin{equation*}
\lambda \sigma_p^{-1} \exp\{ (m+n)(x-p)^2/(2\sigma^2_p) \} =
(1-\lambda) \sigma_q^{-1} \exp\{ (m+n)(x-q)^2/(2\sigma^2_q) \} ,
\end{equation*}
and it can be checked that
$\eta_{n+m,q,p} < \eta_{n,q,p}$ when $m > 1$.
Thus, at the cost of computing the ASE of $\Atilde$,
we may obtain a better estimate.
How much does this additional computation improve
classification the OOS vertices?
Figure~\ref{fig:ratio} explores this question.

\begin{figure}
  \centering
  \includegraphics[width=0.66\columnwidth]{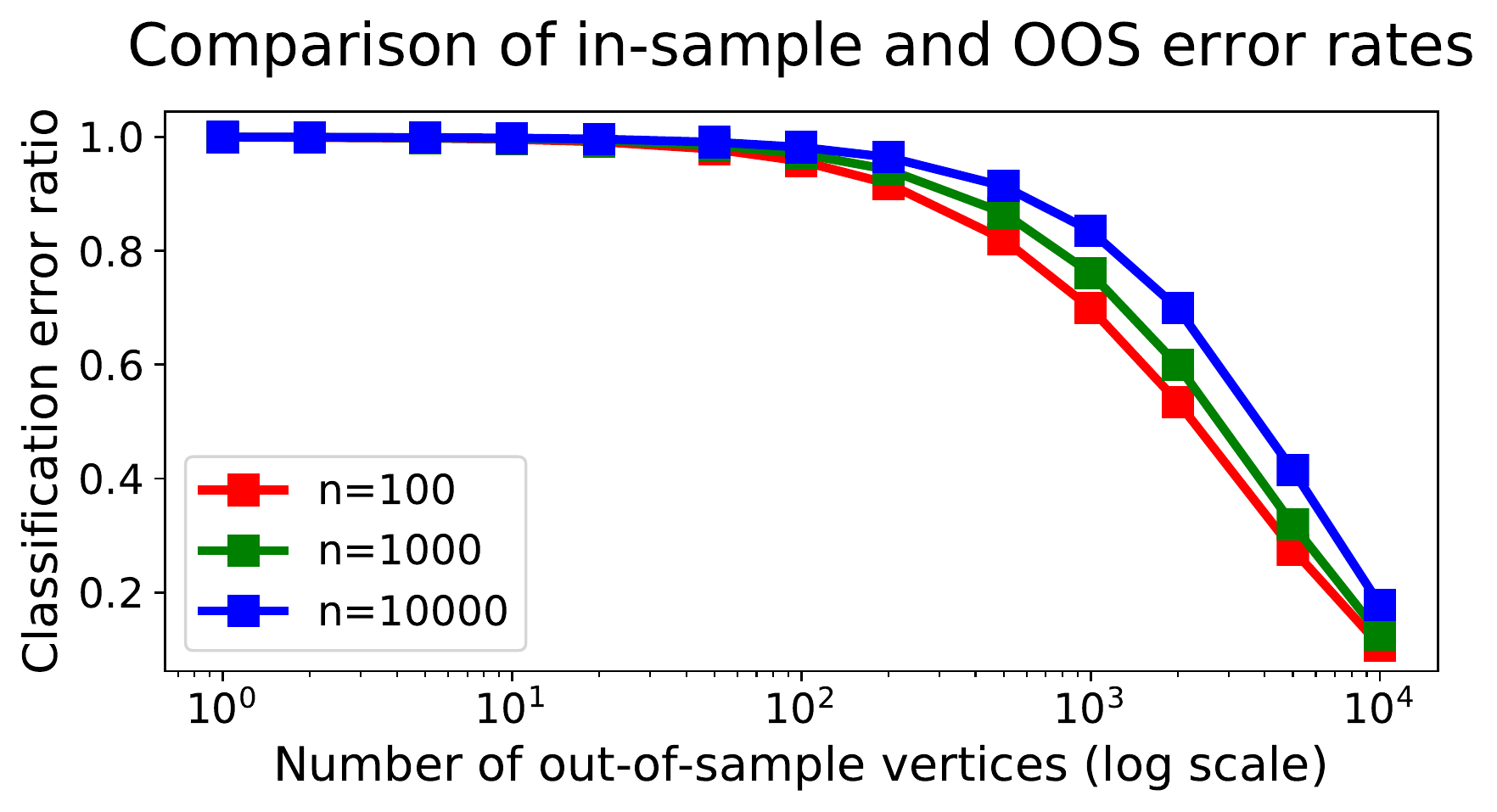}
  \vspace{-5mm}
  \caption{Ratio of the OOS classification error 
	to the in-sample classification error as a function of the
	number of OOS vertices $m$, for $n=100$ vertices,
	$n=1000$ vertices and $n=10000$ vertices. We see that for $m \le 100$,
	the expensive in-sample embedding does not improve appreciably
	on the OOS classification error.
	However, when hundreds or thousands of OOS vertices
	are available simultaneously (i.e., $m \ge 100$), we see that
	the in-sample embedding may improve upon the OOS
	estimate by a significant multiplicative factor.}
  \label{fig:ratio}
\end{figure}

Figure~\ref{fig:ratio} compares the error rates of the in-sample
and OOS estimates as a function of $m$ and $n$ in the model just described,
with $\lambda = 0.4, p=0.6$ and $q=0.61$.
The plot depicts the ratio of the (approximate) in-sample classification
error $\eta_{(n+m),p,q}$ to the (approximate) OOS classification
error $\eta_{(n+1),p,q}$, as a function of the number of OOS
vertices $m$, for differently-sized in-sample graphs,
$n=100, 1000,$ and $10000$.
We see that over several magnitudes of graph size,
the in-sample embedding does not improve appreciably over the
OOS embedding except when multiple hundreds of OOS
vertices are available.
When hundreds or thousands of OOS vertices are available
simultaneously, we see in the right-hand side of Figure~\ref{fig:ratio}
that the in-sample embedding classification error may improve upon
the OOS classification error by a large multiplicative factor.
Whether or not this improvement is worth the additional computational
expense will, depend upon the available resources and desired accuracy,
but this suggests that the additional
expense associated with performing a second ASE computation
is only worthwhile in the event that hundreds or thousands of OOS
vertices are available simultaneously.
This surfeit of OOS vertices is rather divorced from the typical setting
of OOS extension problems, where one typically wishes to embed at most a few
previously unseen observations.

\section{Discussion and Conclusion} \label{sec:conclusion}
We have presented a theoretical investigation of two OOS
extensions of the ASE, one based on a linear least squares estimate
and the other based on a plug-in maximum-likelihood estimate.
We have also proven a central limit theorem for the LLS-based extension,
and simulation suggests that this CLT is a good approximation even with
just a few hundred vertices.
We conjecture that a similar CLT holds for the ML-based OOS extension,
a conjecture supported by similar simulation data.
Finally, we have given a brief illustration of how this OOS extension
and the approximation it introduces
might be weighed against the computational expense of recomputing
a full graph embedding by examining how vertex classification error
depends on the size of the set of OOS vertices.
We leave a more thorough exploration of this trade-off for future work.

\bibliographystyle{plainnat}
\bibliography{biblio}

\newpage
\onecolumn
\appendix
\section*{Appendix}

We collect here the proofs of our two main theorems.
We will make frequent use of the following result,
a proof of which can be found in \citet{AthLyzMarPriSusTan2016} Lemma 1
or in \citet{LevAthTanLyzPri2017} Observation 2.
\begin{lemma}
	[Adapted from \citet{AthLyzMarPriSusTan2016} Lemma 1]
	\label{lem:Psvals}
	Let $X = [X_1,X_2,\dots,X_n]^T \in \R^{n \times d}$ have rows drawn i.i.d.
	from some $d$-dimensional inner product distribution $F$ and denote $ P = XX^{T} $.
	Then with high probability we have $\sigma_1(P) \le n$
	and $\sigma_d(P) \ge cn$.
	Further, it follows that $\sigma_1(X) \le \sqrt{n}$ and
	$\sigma_d(X) \ge c \sqrt{n}$, also with high probability.
\end{lemma}

\section{Proof of Theorem \ref{thm:lsrate}}
\label{sec:proof:lsrate}

To prove Theorem \ref{thm:lsrate}, we must relate the least squares
solution $\whatls$ of \eqref{eq:def:llshat}
to the true latent position $\wtrue$.
We will proceed in two steps.
First, in Section~\ref{subsec:ls}, we will show that
$\whatls$ is close to another least-squares solution $\wls$,
based on the true latent positions $\{ X_i \}$ rather than on the
estimates $\{ \Xhat_i \}$. That is, $\wls$ is the solution
\begin{equation} \label{eq:def:wls}
	\wls = \arg \min_{w \in \R^d} \| X w - \avec \|_F.
\end{equation}
Second, in Section~\ref{subsec:true}, we will show that $\wls$
is close to the true latent position $\wtrue$.
An application of the triangle inequality will then yield our
desired result.

\subsection{Bounding $\| \whatls - \wls \|$}
\label{subsec:ls}

Our goal in this section is to establish a bound on $\| \whatls - \wls \|$,
where $\whatls$ is the solution to Equation~\eqref{eq:def:llshat}
and $\wls$ is as defined by Equation~\eqref{eq:def:wls}.
Our bound will depend upon
a basic result for solutions of perturbed linear systems,
which we adapt from \citet{GolVan2012}.
In essence, we wish to compare
\begin{equation*}
	\whatls = \arg \min_{w \in \R^d} \| \Xhat w - \avec \|_F
\end{equation*}
against 
\begin{equation*}
	\wls = \arg \min_{w \in \R^d} \| X w - \avec \|_F.
\end{equation*}
Recall that for a matrix $B \in \R^{n \times d}$ of full column rank,
we define the condition number
$$ \kappa_2(B) = \frac{ \sigma_1(B) }{ \sigma_d(B) }. $$
\begin{theorem}[\citet{GolVan2012}, Theorem 5.3.1]
	\label{thm:gvl}
	Suppose that $\wls, \rls, \whatls, \rhatls$ satisfy
	$$ \| X \wls - \avec \| = \min_w \| X w - \avec \|,
	\enspace
	\rls = \avec - X \wls, $$
	$$ \| \Xhat \whatls - \avec \| = \min_w \| \Xhat w - \avec \|,
	\enspace
	\rhatls = \avec - \Xhat \whatls, $$
	and that
	\begin{equation} \label{eq:gvl:errorgrowth}
		\| \Xhat - XW \| < \sigma_d( X ).
	\end{equation}
	Assume $\avec, \rls$ and $\wls$ are all non-zero and define
	$\thetals \in (0, \pi/2)$ by
	$\sin \thetals = \| \rls \|/\|\avec\|$.
	If $\epsilon = \|\Xhat - XW \|/\| XW \|$ and
	$$ \nuls = \frac{ \| X \wls \| }{ \sigma_d(XW) \| W^T \wls \| }, $$
	then
	\begin{equation} \label{eq:GVLbound}
		\frac{ \| \whatls - W^T \wls \| }{ \| W^T \wls \| }
		\le \epsilon \left( \frac{ \nuls }{ \cos \thetals }
		+ (1 + \nuls \tan \thetals) \kappa_2(XW) \right)
		+ O(\epsilon^2).
	\end{equation}
\end{theorem}

To apply Theorem~\ref{thm:gvl}, we will first need to show
that the condition in \eqref{eq:gvl:errorgrowth} holds
with high probability, which we show in Lemma~\ref{lem:gvlerrcond}.
We will then show, using Lemma~\ref{lem:gvlerrcond}
and Lemma~\ref{lem:gvlresidual}, that
the right-hand side of \eqref{eq:GVLbound} is also bounded above
by $cn^{-1/2} \log n$ with high probability.

\begin{lemma} \label{lem:gvlerrcond}
	With notation as above, \eqref{eq:gvl:errorgrowth} holds with
	probability at least $1-cn^{-2}$.
	That is, with high probability, there exists an orthogonal matrix
	$W \in \R^{d \times d}$ such that
	\begin{equation} \label{eq:gvlcond:specbound}
		\| \Xhat - XW \| < \sigma_d(XW) .
	\end{equation}
	Further, 
	\begin{equation} \label{eq:gvlcond:epsilon}
		\frac{ \| \Xhat - XW \| }{ \| XW \| } 
		\le \frac{ c \log n }{ \sqrt{n} } \text{ w.h.p. }
	\end{equation}
\end{lemma}
\begin{proof}
	Let $W \in \R^{d \times d}$ be the orthogonal matrix guaranteed by
	Lemma~\ref{lem:perfect}. We begin by observing that
	\begin{equation*}
		\| \Xhat - XW \|^2 \le \| \Xhat - XW \|_F^2
		= \sum_{i=1}^n \| \Xhat_i - W^{T} X_i \|^2
		\le c \log^2 n \text{ w.h.p. },
	\end{equation*}
	where the last inequality follows from Lemma~\ref{lem:perfect}.
	By the construction of the RDPG, we can write
	$XW = \UP \SP^{1/2}W$, from which
	$\sigma_d(XW) = \sigma_d^{1/2}(P) \ge c \sqrt{n}$,
	with the inequality holding high probability by Lemma~\ref{lem:Psvals}.
	This establishes~\eqref{eq:gvlcond:specbound} immediately,
	and it follows that
	$$ \frac{ \| \Xhat - XW \| }{ \| XW \| } \le \frac{ c \log n }{ \sqrt{n} }
	\text{ w.h.p. }, $$
	which proves \eqref{eq:gvlcond:epsilon}.
\end{proof}

\begin{lemma} \label{lem:gvlresidual}
	With notation as in Theorem~\ref{thm:gvl},
	there exists a constant $0 < \gamma \le 1$,
	not depending on $n$, such that
	$\cos \thetals \ge \gamma$ with high probability.
	That is,  there exists a constant $0 \le \gamma' < 1$ such that
	\begin{equation} \label{eq:gvlcond:relativeresidual}
		\sin \thetals = \frac{ \| X \wls - \avec \| }{ \| \avec \| }
		\le \gamma' \text{ w.h.p. }
	\end{equation}
\end{lemma}
\begin{proof}
	To prove \eqref{eq:gvlcond:relativeresidual},
	we begin by noting that, by orthogonality of $ W $, we have $ \wls = \arg \min \| X \wls - \avec \| \Leftrightarrow  W^{T} \wls = \arg \min \| XW w - \avec \|$, hence here~\citet[Theorem 5.3.1]{GolVan2012} has been stated for matrix $ XW $ and its perturbation $ \Xhat $. Furthermore,  since by definition, we have
	$ \| X \wls - \avec \| \le \| X\wtrue - \avec\|$, hence it will suffice for us to show that
	there exists a constant $\gamma \in [0,1)$ for which
	$$ \frac{ \| X \wtrue - \avec \| }{ \| \avec \| } \le \gamma $$
	with high probability, say, with probability of failure $O(n^{-2})$.
	Define the noise vector $\zvec = \avec - X \wtrue$.
	We will show that there exists a constant $\epsilon > 0$, small enough,
	such that
	\begin{equation} \label{eq:sinbound}
		\Pr\left[ \| \avec \|^2 - \| \zvec \|^2 < \epsilon \| \avec \|^2 \right]
		= O( n^{-2} ), \end{equation}
	so that with high probability, $\sin \thetals \le 1-\epsilon$.
	Our argument will proceed in two steps.
	Fix some small $0 < \epsilon < 1$ and let $G_n$ denote the event that
	\begin{equation} \label{eq:Gn}
		\left( \| X\wtrue \|_2
		- \epsilon \frac{ \| X \wtrue \|_1 }{ \| X\wtrue \|_2 }
		\right)
		> 40 \log n.
	\end{equation}
	Defining $Y = \| \avec \|^2 - \| \zvec \|^2 = 2\sum_{i=1}^n a_i X_i^T \wtrue - \| X \wtrue \|^2$,
	let $E_n$ denote the event that
	\begin{equation} \label{eq:En}
		Y - \epsilon \| \avec \|^2
		> \frac{1}{2} \E\left[ Y - \epsilon \| \avec \|^2 \right].
	\end{equation}
	We will show, firstly, that $\Pr( G_n^c ) = O( n^{-2} )$. We further show that $ \E\left( Y - \epsilon \| \avec \|^2 \mid G_{n} \right) > 0 $, which in turn implies that, for large enough $ n $, we have $ \E\left( Y - \epsilon \| \avec \|^2 \right) > 0 $.
	Then, conditioning on the event $G_n$, we will show that
	$\Pr( E_n^c \mid G_n ) = O(n^{-2})$. Finally, denoting the event in~\eqref{eq:sinbound} by $ F_{n} $, we have $ \Pr( F_n^{c} ) \leq \Pr( E_n^{c} ) \leq \Pr( E_n^c \mid G_n ) + \Pr( G_n^c ) = O(n^{-2})$ and our desired result will follow. 
	
	We begin by observing that
	$ \E \| X \wtrue \|^2 = n \E X_1^T \wtrue$,
	since the $X_i^T \wtrue$ are identically distributed.
	Thus, $\E \| X \wtrue \|^2 = \cfw n$, where $0 < \cfw < 1$ is a constant
	that depends only on $\wtrue$ and the latent position distribution $F$.
	Since the $X_i^T \wtrue$ are independent and identically distributed
	(because $\wtrue$ is not random) and 
	$0 \le X_i^T \wtrue \le 1$ almost surely,
	an application of Hoeffding's inequality shows that
	$$ \Pr\left[ \| X\wtrue \|_2^2 - \E \| X \wtrue \|_2^2
	\le -\sqrt{n} \log^{1/2} n \right]
	\le n^{-2}, $$
	whence we have that
	$ \| X \wtrue \|_2^2 \ge \cfw n - \sqrt{n} \log^{1/2} n $
	with high probability, and thus, for suitably large $n$,
	we have $\| X\wtrue \|_2 \ge \sqrt{\cfw  n/2}$ with high probability. Hence, it implies that with high probability
	\begin{equation*} \begin{aligned}
			\| X \wtrue \|_2 - \epsilon \frac{ \| X\wtrue \|_1 }{ \| X \wtrue \|_2 }
			& \ge \sqrt{\frac{\cfw  n}{2}} - \epsilon \sqrt{n} = \left(\sqrt{\frac{\cfw}{2}} - \epsilon\right) \sqrt{n},
	\end{aligned} \end{equation*}
	where $ \epsilon $ is chosen such that $ \epsilon < \sqrt{\cfw/2}$.
	since $\| X \wtrue \|_1 \le \sqrt{n} \| X \wtrue \|_2$ with probability $1$ (by equivalence of norms).
	It follows that for suitably large $n$,
	we have
	\begin{equation} \label{eq:orderlogsqrt}
		\left( \| X \wtrue \|_2 - \epsilon \frac{ \| X\wtrue \|_1 }
		{ \| X \wtrue \|_2 }
		\right)^2
		> 40\log n \text { w.h.p. },
	\end{equation}
	that is, event $G_n$ holds with high probability.
	
	Now, let us condition on this event $G_n$.
	Recalling our definition of
	$Y = 2\sum_{i=1}^n a_i X_i^T \wtrue - \| X \wtrue \|^2$ above. Condition on $ G_n $, i.e., when $ \{X_{i}\}_{i=1}^{n} $ are fixed, we have \begin{align*}
		\E(Y) =  2\sum_{i=1}^n \E(a_i) X_i^T \wtrue - \| X \wtrue \|^2  = 2\sum_{i=1}^n (X_i^T \wtrue)^{2} - \| X \wtrue \|^2 = \| X \wtrue \|^{2},
	\end{align*} 
	and
	\begin{align*}
		\E(\| \avec \|^2 ) = \sum_{i=1}^n \E(a_i^{2}) = \sum_{i=1}^n \E(a_i) = \sum_{i=1}^n X_i^T \wtrue = \| X \wtrue \|_{1}.
	\end{align*}
	Note that on $ G_{n} $, we have $ \E\left( Y - \epsilon \| \avec \|^2 \mid G_{n} \right) > \sqrt{32 \log n} \| X \wtrue \|_2 > 0 $. Furthermore, noting $ a_{i} \in \{0,1\} $, we have
	\begin{align*}
		2 a_i X_i^T \wtrue - (X_{i}^{T} \wtrue )^2 - \epsilon a_{i}^{2} \in \left\{- (X_{i}^{T} \wtrue )^2, 2 X_i^T \wtrue - (X_{i}^{T} \wtrue )^2 - \epsilon\right\}.
	\end{align*}
	Recall that conditioned on $ G_{n} $, we have $ \| X\wtrue \|_{2}^{2} \ge \cfw  n/2$, where $ \cfw   = \E X_1^T \wtrue $. Since $ \epsilon < \sqrt{\cfw/2} $, it follows that, on $ G_{n} $,
	\begin{align*}
		\sum_{i=1}^{n} (2 X_i^T \wtrue - \epsilon)^{2} \leq \sum_{i=1}^{n} 4 (X_i^T \wtrue)^{2} + \epsilon^{2} = 4 \| X \wtrue \|_{2}^{2} + n \epsilon^{2} \leq 5 \| X \wtrue \|_{2}^{2}.
	\end{align*}
	Now, an application of Hoeffding's inequality gives
	\begin{equation*} \begin{aligned}
			\Pr\left[
			Y - \epsilon \| \avec \|^2
			\le \frac{1}{2} \E\left( Y - \epsilon \| \avec \|^2 \right)
			\right]
			&\le \exp\left\{ \frac{ - \left[\E\left( Y - \epsilon \| \avec \|^2 \right) \right]^{2} }
			{ 20 \| X \wtrue \|_2^2 } \right\}
			= \exp\left\{ \frac{
				- \left( \| X \wtrue \|_2^2 - \epsilon \| X\wtrue \|_1 \right)^2 }
			{ 20 \| X \wtrue \|_2^2 } \right\} \\
			&= \exp\left\{ -\frac{ \left( \| X\wtrue \|_2
				- \epsilon \frac{ \| X\wtrue \|_1 }{ \| X\wtrue \|_2 }  \right)^2 }
			{ 20 } \right\}
			\le n^{-2},
	\end{aligned} \end{equation*}
	where the last inequality follows from the fact that event $G_n$ holds.
	Thus, event $E_n$ defined in \eqref{eq:En} and $G_n$ both hold
	with high probability, completing our proof.
\end{proof}

\begin{lemma} \label{lem:ls:hat}
	With notation as in Theorem~\ref{thm:gvl}, with high probability
	there exists orthogonal matrix $W \in \R^{d \times d}$ such that
	$$ \| W \whatls - \wls \| \le c n^{-1/2} \log n . $$
\end{lemma}
\begin{proof}
	This is a direct result of Theorem~\ref{thm:gvl} and the preceding Lemmas,
	once we establish bounds on $\kappa_2(XW)$ and
	$$ \nuls = \frac{ \| XW \wls \| }{ \sigma_d(XW) \| \wls \| }. $$
	By Lemma~\ref{lem:Psvals}, we have $\sigma_d(XW) \ge c\sqrt{n}$ and $\sigma_1(XW) \le \sqrt{n}$, 
	with high probability.
	It follows immediately that $\kappa_2(XW) \le 1/c$ a.s.a.a.
	Further, 
	\begin{align*}
		\nuls &= \frac{ \| X \wls \| }{ \sigma_d(XW) \| W^T \wls \| } = \frac{ \| X W W^T \wls \| }{ \sigma_d(XW) \| W^T \wls \| } =\frac{ \| X W z_{\text{LS}} \| }{ \sigma_d(XW) \| z_{\text{LS}} \| } \leq \frac{ \| X W\|}{ \sigma_d(XW)} = \kappa_2(XW) \leq 1/c, 
	\end{align*}
	with high probability. 
	
	By Lemma~\ref{lem:gvlerrcond}, we are assured that Theorem~\ref{thm:gvl}
	applies a.s.a.a.\ and Lemmas~\ref{lem:gvlerrcond} and \ref{lem:gvlresidual}
	ensure that the each of $(\cos \thetals)^{-1}$ and $\tan \thetals$
	are bounded by constants with high probability.
	Thus, applying Theorem~\ref{thm:gvl} with $\epsilon = cn^{-1/2} \log n$,
	it follows that the right-hand side of
	Equation~\ref{eq:GVLbound} is bounded by $c n^{-1/2} \log n$ with
	high probability, and the result follows.
\end{proof}

\subsection{Bounding $\| \wls - \wtrue \|$ }
\label{subsec:true}
In this section, we will show that $\wls$
is close to the true latent position $\wtrue$.
A combination of this result with Lemma~\ref{lem:ls:hat}
will yield Theorem~\ref{thm:lsrate}.
\begin{lemma} \label{lem:ls:true}
	Condition on $\wtrue$, with high probability, we have
	$$ \| \wls - \wtrue \| \le \frac{ c \log n }{ \sqrt{n} }. $$
\end{lemma}
\begin{proof}
	As noted previously, by definition of $\wls$, we have
	$$ \| XW \wls - \avec \|^2 \le \| X \wtrue - \avec \|^2 = \| \zvec \|^2, $$
	whence plugging in $\avec = X\wtrue + \zvec$ yields
	$ \| XW \wls - X\wtrue - \zvec \|^2 \le \| \zvec \|^2 $.
	Thus,
	\begin{equation} \label{eq:subtractzvec}
		\| XW \wls - X \wtrue \|^2 \le 2\zvec^T X(\wls - \wtrue).
	\end{equation}
	Since $X$ has full column rank,
	it holds with high probability that
	$\| X(W \wls - \wtrue) \| \ge \sigma_d(XW) \| W\wls - \wtrue \|.$
	Combining this fact with \eqref{eq:subtractzvec} and using $\sigma_d^2(X) = \sigma_d(P)$, gives
	$$ \| W\wls - \wtrue \|^2
	\le \frac{ 2\zvec^T X(W\wls - \wtrue) }{ \sigma_d(P) }.$$
	Applying the Cauchy-Schwartz inequality and dividing by $\|W\wls - \wtrue\|$, we obtain
	\begin{align*}
		\| W\wls - \wtrue \| \le \frac{ 2 \| X^T \zvec \| }{ \sigma_d(P) }.
	\end{align*}
	Thus, it remains for us to show that $\| X^T \zvec \|$
	grows as at most $O( \sqrt{n} \log n )$, from which
	Lemma~\ref{lem:Psvals} will yield our desired growth rate.
	Expanding, we have
	\begin{equation} \label{eq:etasum}
		\| X^T \zvec \|_2^2
		= \sum_{j=1}^d \left( \sum_{i=1}^n z_i X_{i,j} \right)^2.
	\end{equation}
	Fixing some $j \in [d]$, note that
	$$ \E z_i X_{i,j} = \E (a_i - X_i^T \wtrue) X_{i,j}
	= 0, $$
	and
	$| z_i X_{i,j}| \le |a_i - X_i^T \wtrue||X_{i,j}| \le 1$,
	so that
	$\{ z_i X_{i,j} \}_{i=1}^n$ is a sum of $n$
	independent bounded zero-mean random variables. A simple application of Hoeffding's inequality thus implies that
	with probability at least $1-O(n^{-2})$,
	$| \sum_{i=1}^n z_i X_{i,j} | \le 2 \sqrt{n} \log n$.
	A union bound over all $d$ sums in Equation~\eqref{eq:etasum},
	since $d$ is assumed to be constant in $n$,
	we have that
	$\| X^T \zvec \|_2^2 \le 4dn \log^2 n$
	with high probability,
	and taking square roots completes the proof.
\end{proof}

\section{Proof of Theorem \ref{thm:mlrate}}
\label{sec:proof:mlrate}

We remind the reader that $\whatml$ denotes the
optimal solution
$$ \whatml = \arg \max_{w \in \calThat_\epsilon} \ellhat(w), $$
where
$\ellhat(w) = \sum_{i=1}^n a_i \log \Xhat_i^T w +(1-a_i) \log (1-\Xhat_i^T w).$
To prove Theorem~\ref{thm:mlrate}, we will
apply a standard argument
from convex optimization and use the properties of the set
$\calThat_\epsilon$ to show that
for suitably large $n$,
$$ \| W \whatml - \wtrue \|
\le \frac{ \| \nabla \ellhat( W^T \wtrue ) \| }{ c n }, $$
where $W \in \R^{d \times d}$ is the orthogonal matrix guaranteed
by Lemma~\ref{lem:perfect}. This is proven in Lemma~\ref{lem:cvxopt}.
We then show in Lemma~\ref{lem:nabla} that
$$\| \nabla \ellhat( W^T \wtrue ) \| \le c \sqrt{n} \log n \text{ w.h.p. } $$
Combining these two facts establishes the theorem.

\begin{lemma} \label{lem:cvxopt}
	With notation as above, under the assumptions of Theorem~\ref{thm:mlrate},
	for $n$ suitably large, there exists an orthogonal matrix $W \in \R^{d \times d}$ such that with probability at least $1 - cn^{-2}$,
	$$ \| W \whatml - \wtrue \|
	\le \frac{ \| \nabla \ellhat( W^T \wtrue ) \| }{ c n }. $$
\end{lemma}
\begin{proof}
	We begin by noting that $\ellhat(w)$ is convex in its argument, and that $\whatml$ is the solution to a convex constrained optimization problem.
	Thus, by the optimality condition for convex constrained problems, along with the mean value theorem for vector-valued functions, we have
	\begin{align*} 
		\left( \nabla \ellhat(W^T \wtrue) \right)^T (W^T \wtrue - \whatml ) &= \left( \nabla \ellhat(W^T \wtrue) - \nabla \ellhat(\whatml) + \nabla \ellhat(\whatml) \right)^T (W^T \wtrue - \whatml ) \\
		&= \left( \nabla \ellhat(\whatml) \right)^T (W^T \wtrue - \whatml) \\
		&~~~~~~+ \int_0^1 (W^T \wtrue - \whatml)^T
		\nabla^2 \ellhat\left( W^T \wtrue + t(W^T \wtrue - \whatml) \right)
		(W^T \wtrue- \whatml) dt \\
		&\ge \| W^T \wtrue - \whatml \|^2
		\min_{w \in \calThat_\epsilon}
		\lambdamin \left( \nabla^2 \ellhat(w) \right).
	\end{align*}
	The constraint that $w \in \calThat_\epsilon$ implies that
	for suitably large $n$,
	$$ \min_{w \in \calThat_\epsilon}
	\lambdamin\left( \nabla^2 \ellhat(w) \right)
	\ge c n, $$
	with $c > 0$ depending on $\epsilon$ but not on $n$.
	By unitary invariance of the Euclidean norm and the Cauchy-Schwarz inequality, it follows that
	$$ \| \wtrue - W\whatml \|
	= \| W^T \wtrue - \whatml \|
	\le \frac{ \| \nabla \ellhat( W^T \wtrue ) \| }{ c n }, $$
	completing the proof.
\end{proof}

\begin{lemma} \label{lem:nabla}
	With notation as above,
	under the assumptions of Theorem~\ref{thm:mlrate},
	for all suitably large $n$,
	with probability at least $1-cn^{-2}$,
	$$ \| \nabla \ellhat( W^T \wtrue ) \| \le c \sqrt{n} \log n . $$
\end{lemma}
\begin{proof}
	Let $W \in \R^{d \times d}$ be the orthogonal matrix
	whose existence is guaranteed by Lemma~\ref{lem:perfect},
	and denote by $W^T \supp F$ the set
	$\{ W^T x : x \in \supp F \}$.
	Analogously to $\ellhat(w)$,
	define $\ell : W^T \supp F \rightarrow \R$ by
	$$ \ell(w) = \sum_{i=1}^n a_i \log X_i^T W w
	+(1-a_i) \log (1-X_i^T W w). $$
	We involve $W$ in this function so that we may think of
	$\ellhat$ and $\ell$ as operating on the same set, with $W^T$ serving to rotate the support of $F$ to (approximately) agree with the estimates $\{ \Xhat_i \}_{i=1}^n$.
	
	By the triangle inequality,
	\begin{equation} \label{eq:gradtriangle}
		\| \nabla \ellhat( W^T\wtrue ) \|
		\le \| \nabla \ell( W^T\wtrue ) \| +
		\| \nabla \ellhat( W^T \wtrue ) - \nabla \ell( W^T \wtrue ) \|.
	\end{equation}
	We will show that
	both terms on the right hand side of \eqref{eq:gradtriangle}
	are bounded by $c \sqrt{n} \log n$ with high probability.
	
	Fix $j \in [d]$. We observe first that, conditioning on
	$\wtrue$ and $\{ X_i \}_{i=1}^n$,
	$$ (\nabla \ell( W^T \wtrue ) )_j = 
	\sum_{i=1}^n \left( \frac{ a_i }{ X_i^T \wtrue }
	- \frac{ 1-a_i }{ 1 - X_i^T \wtrue } \right) (X W)_{i,j}
	= \sum_{i=1}^n \frac{ (a_i - X_i^T \wtrue) ( XW)_{i,j} }
	{ X_i^T \wtrue (1-X_i^T \wtrue) } $$
	is a sum of zero-mean random variables, each of which is bounded owing to our assumption that $\supp F$ is bounded away from 0 and 1.
	Applying Hoeffding's inequality,
	$$ \Pr\Big[ \left| (\nabla \ell( W^T \wtrue ) )_j \right| \ge t \Big] 
	\le 2\exp\left\{ \frac{ -2t^2 }{ c n } \right\} $$
	for some suitably-chosen constant $c$ depending on $F$.
	Choosing $t = \sqrt{ cn } \log n$,
	we have $(\nabla \ell( W^T \wtrue ) )_j \ge \sqrt{ cn } \log n$
	with probability at most $2n^{-2}$.
	A union bound over all $j \in [d]$, implies that
	with probability at least $1 - 2dn^{-2}$,
	$$ \sum_{j=1}^d \left( \nabla \ell( W^T \wtrue ) \right)^2_j
	\le d c n \log^2 n, $$
	whence an application of the Borel-Cantelli Lemma
	yields that $\| \nabla \ell( W^T \wtrue ) \| = O(\sqrt{n} \log n)$ almost surely.
	
	Turning to the second term on the right-hand side of
	\eqref{eq:gradtriangle}, fixing $j \in [d]$, we have
	\begin{equation*}
		\left( \nabla \ellhat( W^T \wtrue ) - \nabla \ell( W^T \wtrue ) \right)_j
		= \sum_{i=1}^n \frac{ (a_i - \Xhat_i^T W^T \wtrue) \Xhat_{i,j} }{ \Xhat_i^T W^T \wtrue (1-\Xhat_i^T W^T \wtrue) }
		- \sum_{i=1}^n \frac{ (a_i - X_i^T \wtrue) (XW)_{i,j} }
		{ X_i^T \wtrue (1-X_i^T \wtrue) }.
	\end{equation*}
	Taking expectations, we have
	\begin{equation} \label{eq:nablaexpec}
		\E\left[ \left( \nabla \ellhat( W^T \wtrue ) - \nabla \ell( W^T \wtrue ) \right)_j \right]
		= \sum_{i=1}^n \frac{ ( (X_i - W \Xhat_i)^T \wtrue) \Xhat_{i,j} }{ \Xhat_i^T W^T \wtrue (1-\Xhat_i^T W^T \wtrue) }.
	\end{equation}
	Conditioned on $\wtrue$ and the latent positions $\{ X_i \}$,
	\eqref{eq:nablaexpec}
	is a sum of $n$ terms.
	By Lemma~\ref{lem:perfect},
	with high probability, all $n$ of these terms
	are bounded by $c n^{-1/2} \log n $.
	Call this bounding event $B$.
	Then, taking expectation conditional on $B$,
	$$ \E\Big[ \big( \nabla \ellhat( \wtrue ) - \nabla \ell( \wtrue )
	\big)_j \Big| B \Big]
	= O( \sqrt{n} \log n ) \text{ a.s. } $$
	Our proof will be complete if we can show that with high probability,
	$\nabla \ellhat( \wtrue ) - \nabla \ell( \wtrue )$
	concentrates about its mean with a deviation that is
	at most $c n^{1/2} \log n $.
	
	Keeping $j \in [d]$ fixed, define the quantities
	$p_i = X_i^T \wtrue$ and
	$\phat_i = \Xhat_i^T W^T \wtrue $
	so that
	\begin{equation*} \begin{aligned}
			&\left( \nabla \ellhat( W^T \wtrue ) - \nabla \ell( W^T \wtrue ) \right)_j
			- \E\Big[ \big( \nabla \ellhat( \wtrue ) - \nabla \ell( \wtrue ) \big)_j \Big] \\
			&~~~= \sum_{i=1}^n
			\frac{ (a_i - \phat_i)\Xhat_{i,j} }
			{ \phat_i(1-\phat_i) }
			- \frac{ (a_i - p_i) (XW)_{i,j} }{ p_i (1-p_i) }
			- \frac{ (p_i - \phat_i) \Xhat_{i,j} }
			{ \phat_i(1-\phat_i) } \\
			&~~~= \sum_{i=1}^n
			(a_i - p_i)
			\left( \frac{ \Xhat_{i,j} }{ \phat_i(1-\phat_i) }  
			-
			\frac{ (XW)_{i,j} }{ p_i (1-p_i) } \right)
	\end{aligned} \end{equation*}
	We note that by Lemma~\ref{lem:perfect}
	and our boundedness assumption on $\supp F$,
	we have that for suitably large $n$,
	with high probability it holds for all $i\in[n]$ that
	\begin{equation*}
		|\phat_i - p_i| \le c n^{-1/2} \log n,~
		|\Xhat_{i,j} - (XW)_{i,j}| \le c n^{-1/2} \log n,
		\text{ and }
		\left| \frac{1}{ \phat_i(1-\phat_i)} \right|
		\le \frac{4}{\epsilon}.
	\end{equation*}
	Thus, with high probability, for all $i \in [n]$,
	\begin{equation*} \begin{aligned}
			&\left| (a_i - p_i)
			\left( \frac{ \Xhat_{i,j} }{ \phat_i(1-\phat_i) }  
			-
			\frac{ (XW)_{i,j} }{ p_i (1-p_i) } \right) \right|
			\le \left| \frac{ \Xhat_{i,j} }{ \phat_i(1-\phat_i) }  
			-
			\frac{ (XW)_{i,j} }{ p_i (1-p_i) } \right| \\
			&~~~~~~\le \frac{ | \Xhat_{i,j} - (XW)_{i,j} |}
			{ \phat_i (1-\phat_i) }
			+ |(XW)_{i,j}|
			\left| \frac{ 1}{\phat_i(1-\phat_i)}
			- \frac{1}{p_i(1-p_i)} \right|.
	\end{aligned} \end{equation*}
	By Lemma~\ref{lem:perfect},
	the first of these terms is bounded with high probability by
	$ 4c\epsilon^{-1}n^{-1/2}\log n,$
	and since $\epsilon$ is a constant,
	we have that this first term is bounded by
	$cn^{-1/2}\log n$.
	The second term is similarly bounded,
	since $| (W^T X)_{i,j} | \le 1$
	by $X \in \supp F$ and
	$$
	\left| \frac{ 1}{\phat_i(1-\phat_i)}
	- \frac{1}{p_i(1-p_i)} \right|
	\le \frac{ |\phat_i - p_i | + |\phat_i^2 - p_i^2| }{ p_i(1-p_i)\phat_i(1-\phat_i) }
	= \frac{ |\phat_i - p_i |(1 + |\phat_i + p_i|) }
		{ p_i(1-p_i)\phat_i(1-\phat_i) }
	\le c n^{-1/2} \log n, $$
	since $\phat_i$ is bounded away from $0$ and $1$ and
	$|\phat_i - p_i| \le c n^{-1/2} \log n$,
	both with high probability by Lemma~\ref{lem:perfect}
	and our boundedness assumptions on $\supp F$.
	
	Thus, we have shown that both terms in Equation~\eqref{eq:gradtriangle} grow as
	$O( \sqrt{n} \log n )$ almost surely,
	which proves the theorem.
\end{proof}

\section{Proof of Theorem~\ref{thm:clt}} \label{apx:clt}

In this section, we will prove the central limit theorem
presented in Theorem~\ref{thm:clt},
which shows that for a suitably-chosen sequence of orthogonal matrices
$\{ V_n \}_{n=1}^\infty$, the quantity $\sqrt{n}(V_n^T \whatls - \wtrue)$
is asymptotically multivariate normal.
We begin by recalling that
$$ \whatls = (\Xhat^T \Xhat)^{-1} \Xhat^T \avec = \SA^{-1/2} \UA^T \avec. $$
Our proof of Theorem~\ref{thm:clt} will consist of writing
$\sqrt{n}(\whatls - V\wtrue)$ as a sum of two random vectors,
$$ \sqrt{n}(\whatls - V\wtrue)
= \sqrt{n} \gvec + \sqrt{n} \hvec, $$
and showing that for suitable choice of $V$,
$\sqrt{n}\gvec$ converges in law to a normal,
while $\sqrt{n}\hvec$ converges in probability to $0$,
from which the multivariate version of Slutsky's Theorem
will yield the desired result.
We begin by showing that
$\gvec = \sqrt{n} \SP^{-1/2} \UP^T(\avec - X \wtrue)$ will suffice.

\begin{lemma} \label{lem:inlaw}
	Let $(A,X) \sim \RDPG(F,n)$, notation as above, etc.
	$$ \sqrt{n} \SP^{-1/2} \UP^T(\avec - X \wtrue)
	\inlaw \calN(0, \Sigma_{\wtrue} ), $$
	where
	$\Sigma_{\wtrue} =
	\Delta^{-1} \E\left[X_1^T \wtrue(1-X_1^T\wtrue) X_1 X_1^T \right] \Delta^{-1}$.
\end{lemma}
\begin{proof}
	We begin by observing that
	\begin{equation*}
		n^{-1/2} X^T(\avec - X \wtrue)
		= \frac{1}{n^{1/2}} \sum_{i=1}^n (\avec_i - X_i^T \wtrue) X_i
	\end{equation*}
	is a scaled sum of of $n$ independent $0$-mean
	$d$-dimensional random vectors, each with covariance matrix
	$$ \Sigmatilde_{\wtrue} = \E X_1^T \wtrue(1-X_1^T\wtrue) X_1 X_1^T. $$
	The multivariate central limit theorem implies that
	$$ n^{-1/2} X^T(\avec - X \wtrue) X_i \inlaw \calN(0,\Sigmatilde_{\wtrue}). $$
	We have
	$\sqrt{n} \SP^{-1/2} \UP^T( \avec - X \wtrue)
	= n \SP^{-1} n^{-1/2} X^T( \avec - X \wtrue )$.
	By the WLLN, $\SP/n \inprob \Delta$, and hence
	by the continuous mapping theorem, $n \SP^{-1} \inprob \Delta^{-1}$.
	Thus, the multivariate version of Slutsky's Theorem implies that
	$$ \sqrt{n} \SP^{-1/2} \UP^T( \avec - X \wtrue) \inlaw
	\calN( 0, \Delta^{-1} \Sigmatilde_{\wtrue} \Delta^{-1}), $$
	as we set out to show.
\end{proof}

The statement of Theorem~\ref{thm:clt} asserts the existence of a sequence of
orthogonal matrices $V_n \in \R^{d \times d}$.
It will turn out that the appropriate matrix is given by
\begin{equation} \label{eq:def:V}
	V_n = V_A V_P^T,
\end{equation}
where $V_A \Sigma V_P^T$ is the SVD of $\UA^T \UP$.
In what follows, we will drop the dependence on $n$ for ease of notation,
but we remind the reader that all quantities
are assumed to depend on $n$
aside from the distribution $F$ and dimension $d$.

\subsection{Technical Lemmas}

The proof of Theorem~\ref{thm:clt} relies on several bounds relating the
matrices $\UA,\UP$ and $V$ developed in \citet{lyzinski15_HSBM},
which we collect here.

\begin{lemma}[Adapted from \citet{lyzinski15_HSBM}, Proposition 16].
	\label{lem:HSBM:SVD}
	With $V \in \R^{d \times d}$ as defined in Equation~\eqref{eq:def:V},
	we have
	$$ \| \UA^T \UP - V \|_F \le \frac{ c \log n }{ n }. $$
\end{lemma}

\begin{lemma}[\citet{lyzinski15_HSBM}, Lemma 17]
	\label{lem:approxcommute}
	Let $(A,X) \sim \RDPG(F)$, and let $V$ be as defined
	in Equation~\eqref{eq:def:V}.
	The following two bounds hold with high probability:
	$$ \| \SA^{-1/2}V - V\SP^{-1/2} \|_F \le \frac{ c \log n }{ n^{3/2} } $$
	and
	$$ \| \SA^{1/2}V - V\SP^{1/2} \|_F \le \frac{ c \log n }{ n^{1/2} }. $$
\end{lemma}
%

We will also need the following result, which is a basic
application of Hoeffding's inequality.

\begin{lemma} \label{lem:unitaryhoeff}
	With notation as above, with high probability,
	$$ \| \UA^T(\avec - X\wtrue) \|_F \le c n^{1/2} \log n. $$
\end{lemma}
\begin{proof}
	For $j \in [d]$ and $i \in [n]$, observe that
	\begin{equation*}
		\left( \UA^T(\avec - X\wtrue) \right)_{j,i}
		= \sum_{k=1}^n (\UA)_{k,j}(a_k - X_k^T \wtrue)
	\end{equation*}
	is a sum of independent $0$-mean random variables,
	and Hoeffding's inequality yields
	$$ \Pr\left[ | \UA^T(\avec - X\wtrue) |_{j,i} \ge t \right]
	\le 2\exp\left\{ \frac{ -t^2 }{ 2\sum_{k=1}^n (\UA)^2_{j,k} } \right\}
	= 2\exp\left\{ \frac{ -t^2 }{ 2 } \right\}. $$
	Taking $t = 3\sqrt{2} \log n$ and a union bound over the
	$nd$ entries of $\UA^T(\avec - X\wtrue)$ yields the result.
\end{proof}

The following spectral norm bound will be useful at several points in
our proof of Theorem~\ref{thm:clt}.
\begin{theorem}[Matrix Bernstein inequality, \citet{Tropp2015}]
	\label{thm:asymbern}
	Let $\{ Z_k \}$ be a finite collection
	of random matrices in $\R^{d_1 \times d_2}$
	with $\E Z_k = 0$ and $\| Z_k \| \le R$ for all $k$, then
	$$ \Pr\left[ \Big\| \sum_k Z_k \Big\| \ge t \right]
	\le (d_1 + d_2)\exp\left\{ \frac{ -t^2 }{ \nu^2 + Rt/3 } \right\}, $$
	where
	$$ \nu^2 = \max\left\{ \Big\| \sum_k \E Z_k Z_k^T \Big\|,
	\Big\| \sum_k \E Z_k^T Z_k \Big\| \right\}. $$
\end{theorem}

The following technical lemma will be crucial for proving one of the
convergences in probability required by our main theorem.
Its comparative complexity merits stating it here rather than
including it in the proof of Theorem~\ref{thm:clt} below.
\begin{lemma} \label{lem:exchangeable}
	With notation as above,
	$$ \sqrt{n}\SA^{-1/2}(\UA^T - \UA^T\UP \UP^T)(\avec - X \wtrue)
	\inprob 0 . $$
\end{lemma}
\begin{proof}
	For ease of notation, define the vector
	$$ \zvec = (\UA^T - \UA^T \UP \UP^T)(\avec - X \wtrue ) .$$
	Define the matrix
	$$ \Wtilde = \onevec_n^T \otimes \wtrue
	= \begin{bmatrix} \wtrue \wtrue \dots \wtrue \end{bmatrix}
	\in \R^{d \times n}. $$
	Let $\Btilde \in \R^{n \times n}$ be a random matrix with
	independent binary entries with
	$\E \Btilde_{i,j} = (X \Wtilde)_{i,j} = X_i^T \wtrue.$
	Define the events
	$$ G_1 = \left\{ \Big\| \SA^{-1/2} \Big\| \le n^{-1/2} \right\}, $$
	$$ G_2 = \left\{ \Big\| (\UA^T - \UA^T\UP\UP^T)(\Btilde-X\Wtilde) \Big\|_F^2
	\le c \log^2 n \right\},$$
	and
	$$ G_3 = \{ \sqrt{n} \| \zvec \| \ge n^{1/4} \}. $$
	It is clear that when events $G_1$ and $G_3$ both occur, we have
	$$ \left\| \sqrt{n}\SA^{-1/2}(\UA^T - \UA^T\UP \UP^T)(\avec - X \wtrue) \right\|
	\le c n^{-1/4} . $$
	By Lemma~\ref{lem:Psvals}, event $G_1$ occurs with high probability,
	so the proof will be complete if we can show that
	\begin{equation} \label{eq:G3lim}
		\lim_{n \rightarrow \infty} \Pr[ G_3^c ] = 0.
	\end{equation}
	To do this, we will require a slightly more involved argument.
	We note that
	$$ \Pr[ G_3^c ] \le \Pr[ G_3^c \mid G_2]\Pr[G_2]
	+ \Pr[G_2^c]. $$
	To show \eqref{eq:G3lim}, it will suffice to show that
	\begin{enumerate}
		\item $G_2$ occurs with high probability, and
		\item $\lim_{n\rightarrow \infty} \Pr[ G_3^c \mid G_1,G_2] = 0$.
	\end{enumerate}
	
	By submultiplicativity, we have
	\begin{equation} \label{eq:Zsubmult}
		\| (\UA^T - \UA^T\UP\UP^T)(\Btilde-X\Wtilde) \|_F^2
		\le \| \UA^T - \UA^T\UP\UP^T \|_F^2 \| \Btilde - X\Wtilde \|^2.
	\end{equation}
	Theorem~\ref{thm:asymbern} applied to $\Btilde - X \Wtilde$
	implies that with high probability,
	\begin{equation} \label{eq:asymbern}
		\| \Btilde - X \Wtilde \| \le c n^{1/2} \log^{1/2} n.
	\end{equation}
	The Davis-Kahan Theorem \citep{DavKah1970,Bhatia1997} shows that
	$$ \| \UA \UA^T - \UP \UP^T \| \le \frac{ c \| A - P \| }{ \lambda_d(P) }, $$
	while Theorem 2 in \citet{YuWanSam2015} shows that there exists orthonormal
	$W^* \in \R^{d \times d}$ such that
	$$ \| \UA - \UP W^* \|_F \le c \| \UA \UA^T - \UP \UP^T \|)F
	\le \frac{ c \log^{1/2} }{ n^{1/2} }. $$
	Since $W = \UA^T \UP$ solves the minimization
	$$ \min_{W \in \R^{d \times d} } \| \UA^T W - \UA^T \UP \UP^T \|_F, $$
	we have that
	\begin{equation*}
		\begin{aligned}
			\| \UA^T - \UA^T \UP \UP^T \|_F^2 
			&\le \| \UA^T - W^* \UP^T \|_F^2
			\le c \| \UA \UA^T - \UP \UP^T \|_F^2
			\le \frac{ c \| A - P \| }{ \lambda_d(P) }
			\le \frac{ c \log^{1/2} n }{ \sqrt{n} } \text{ w.h.p., }
		\end{aligned}
	\end{equation*}
	where the last inequality follows from
	an application of Lemma~\ref{lem:Psvals} and
	the matrix Bernstein inequality applied to $\| A - P \|$.
	Plugging this and~\eqref{eq:asymbern} back into~\eqref{eq:Zsubmult}, we have
	\begin{equation} \label{eq:Zbound}
		\| (\UA^T - \UA^T\UP\UP^T)(\Btilde-X\Wtilde) \|_F^2
		\le c \log^2 n \text{ w.h.p., }
	\end{equation}
	which is to say, $G_2$ occurs with high probability.
	
	It remains to show that
	$\lim_{n\rightarrow \infty} \Pr[ G_3^c \mid G_2] = 0$.
	By construction, the columns of matrix
	$(\UA^T - \UA^T\UP \UP^T)(\Btilde - X\Wtilde)$
	are $n$ independent copies of $\zvec$.
	Using this fact and the Markov inequality, we have
	$$ \Pr[G_3^c \mid G_2 ] = \Pr[ \sqrt{n} \| \zvec \| \ge n^{1/4} \mid G_2 ]
	\le \frac{ n\E[\| \zvec \|^2 \mid G_2] }{ n^{1/2} }
	= \frac{ \E[ \| (\UA^T - \UA^T\UP \UP^T)(\Btilde - X\Wtilde) \|_F^2
		\mid G_2 ] }{ n^{1/2} }
	\le \frac{ c \log^2 n }{ n^{1/2} }, $$
	where the last inequality follows from the definition of event $G_2$.
	This quantity goes to zero in $n$, which completes the proof.
\end{proof}

\subsection{Theorem~\ref{thm:clt} proof details}
We are now ready to present the proof of Theorem~\ref{thm:clt}.

\begin{proof}[Proof of Theorem~\ref{thm:clt}]
	Adding an subtracting appropriate quantities,
	\begin{equation*}
		\begin{aligned}
			\sqrt{n}(V^T \whatls - \wtrue)
			&= \sqrt{n} V^T \left( \SA^{-1/2} \UA^T \avec - V \wtrue \right) \\
			&= \sqrt{n}\SP^{-1/2} \UP^T (\avec - X\wtrue ) \\
			&~~~~~~ + \sqrt{n} V^T \SA^{-1/2}(\UA^T - V \UP^T)(\avec - X \wtrue ) \\
			&~~~~~~ + \sqrt{n} V^T ( \SA^{-1/2} \UA^T X - V)\wtrue \\
			&~~~~~~ + \sqrt{n} V^T (\SA^{-1/2}V - V\SP^{-1/2})\UP^T(\avec-X\wtrue).
		\end{aligned}
	\end{equation*}
	Our proof will consist of showing that the first of these terms
	goes in law to a normal, and that the remaining terms go to zero in
	probability, from which the multivariate version of Slutsky's Theorem
	will imply our desired convergence in law.
	By Lemma~\ref{lem:inlaw},
	\begin{equation} \label{eq:inlaw:restate}
		\sqrt{n} \SP^{-1/2} \UP^T (\avec - X\wtrue )
		\inlaw \calN(0, \Sigma_{\wtrue}),
	\end{equation}
	where $\Sigma_{\wtrue}$ is as defined in Lemma~\ref{lem:inlaw}.
	Thus, the first term in our expansion of
	$\sqrt{n}(V^T \whatls - \wtrue)$ converges in distribution
	as required.
	
	Since $V$ is orthogonal,
	it will suffice to prove the following three convergences in probability:
	\begin{equation} \label{eq:inprob:1}
		\sqrt{n} \SA^{-1/2}(\UA^T - V \UP^T)(\avec - X \wtrue ) \inprob 0,
	\end{equation}
	\begin{equation} \label{eq:inprob:2}
		\sqrt{n}( \SA^{-1/2} \UA^T X - V)\wtrue \inprob 0,
	\end{equation}
	and
	\begin{equation} \label{eq:inprob:3}
		\sqrt{n}(\SA^{-1/2}V - V\SP^{-1/2})\UP^T(\avec-X\wtrue) \inprob 0.
	\end{equation}
	We will address each of these three convergences in order.
	
	To see the convergence in~\eqref{eq:inprob:1},
	adding and subtracting appropriate quantities gives
	\begin{equation} \label{eq:inprob1:split}
		\sqrt{n} \SA^{-1/2}(\UA^T - V \UP^T)(\avec - X \wtrue )
		= \sqrt{n}\SA^{-1/2}(\UA^T \UP \UP^T - V \UP^T)(\avec - X \wtrue)
		+ \sqrt{n}\SA^{-1/2}(\UA^T - \UA^T\UP \UP^T)(\avec - X \wtrue).
	\end{equation}
	To bound the first of these two summands, note that a union bound over
	the events of Lemmas~\ref{lem:Psvals} and \ref{lem:HSBM:SVD}
	and an argument identical to that in Lemma~\ref{lem:unitaryhoeff} yields that
	$$ \| \sqrt{n}\SA^{-1/2}(\UA^T \UP \UP^T - V \UP^T)(\avec - X \wtrue) \|
	\le \sqrt{n} \| \SA^{-1/2} \| \| \UA^T \UP - V \| \| \UP^T(\avec-X\wtrue) \|
	\le c n^{-1/2} \log^2 n \text{ w.h.p. } $$
	Lemma~\ref{lem:exchangeable} shows that the second term
	in~\eqref{eq:inprob1:split} also goes to zero in probability.
	
	To see~\eqref{eq:inprob:2}, note that
	\begin{equation*}
		\begin{aligned}
			\sqrt{n}(\SA^{-1/2}\UA^T X - V)\wtrue
			&= \sqrt{n}\left( \SA^{-1/2}\UA^T \UP \SP^{1/2} - V \right) \wtrue \\
			&= \sqrt{n}\SA^{-1/2}\left( \UA^T\UP - V\right) \SP^{1/2}\wtrue
			+ \sqrt{n}\SA^{-1/2}\left( V \SP^{1/2} - \SA^{1/2} V \right)\wtrue.
		\end{aligned}
	\end{equation*}
	Submultiplicativity of matrix norms combined with Lemmas~\ref{lem:Psvals}
	and~\ref{lem:HSBM:SVD} and the fact that $\| \wtrue \| \le 1$
	imply that with high probability,
	\begin{equation} \label{eq:inprob2:tri1}
		\| \sqrt{n}\SA^{-1/2}\left( \UA^T\UP - V\right) \SP^{1/2} \wtrue \|
		\le c \| \UA^T \UP - V \|_F \| SP^{1/2} \| \| \wtrue \|
		\le c n^{-1/2} \log n.
	\end{equation}
	Applying Lemma~\ref{lem:Psvals} again and taking the Frobenius norm as
	a trivial upper bound on the spectral norm,
	Lemma~\ref{lem:HSBM:SVD} implies
	\begin{equation} \label{eq:inprob2:tri2}
		\| \sqrt{n}\SA^{-1/2}\left( V \SP^{1/2} - \SA^{1/2} V \right) \wtrue \|
		\le c \| V \SP^{1/2} - \SA^{1/2} V \| \| \wtrue \|
		\le c  c n^{-1/2} \log n \text{ w.h.p. }
	\end{equation}
	Combining Equations~\eqref{eq:inprob2:tri1} and~\eqref{eq:inprob2:tri2}
	proves \eqref{eq:inprob:2} by the triangle inequality.
	
	Finally, to prove~\eqref{eq:inprob:3}, note that
	\begin{equation*}
	\| \sqrt{n}(\SA^{-1/2}V - V\SP^{-1/2})\UP^T(\avec-X\wtrue) \|
	\le \sqrt{n} \| \SA^{-1/2}V - V\SP^{-1/2} \| \| \UP^T(\avec - X\wtrue) \|_F,
	\end{equation*}
	and Lemmas~\ref{lem:approxcommute} and~\ref{lem:unitaryhoeff}
	imply that
	$$ \| \sqrt{n}(\SA^{-1/2}V - V\SP^{-1/2})\UP^T(\avec-X\wtrue) \|
	\le c n^{-1/2} \log n \text{ w.h.p., } $$
	and the result is proved.
\end{proof}

\end{document}

%% file: newcommands.tex
\usepackage{amsmath,amssymb,amsthm}
\usepackage{enumerate}
\usepackage{appendix}
\usepackage{graphicx}
\usepackage{color}
\usepackage{hyperref}
\hypersetup{
	colorlinks   = true,
	linkcolor    = blue,
	citecolor    = green
}

\newtheorem{theorem}{Theorem}
\newtheorem{corollary}{Corollary}
\newtheorem{lemma}{Lemma}
\newtheorem{definition}{Definition}
\newtheorem{remark}{Remark}

\newcommand{\R}{\mathbb{R}}
\newcommand{\Rnonneg}{\mathbb{R}_{\ge 0}}

\newcommand{\Znonneg}{\mathbb{Z}_{\ge 0}}
\newcommand{\E}{\mathbb{E}}

\newcommand{\calD}{\mathcal{D}}
\newcommand{\calL}{\mathcal{L}}
\newcommand{\calN}{\mathcal{N}}
\newcommand{\calS}{\mathcal{S}}
\newcommand{\calT}{\mathcal{T}}
\newcommand{\calThat}{\widehat{\calT}}
\newcommand{\calX}{\mathcal{X}}

\newcommand{\ellhat}{\hat{\ell}}
\newcommand{\phat}{\hat{p}}
\newcommand{\rhat}{\hat{r}}
\newcommand{\what}{\hat{w}}
\newcommand{\wtrue}{\bar{w}}
\newcommand{\Xhat}{\hat{X}}

\newcommand{\Atilde}{\tilde{A}}
\newcommand{\Btilde}{\tilde{B}}
\newcommand{\calDtilde}{\calD}
\newcommand{\Gtilde}{\tilde{G}}
\newcommand{\Wtilde}{\tilde{W}}
\newcommand{\wtilde}{\tilde{w}}
\newcommand{\Xtilde}{\tilde{X}}
\newcommand{\Sigmatilde}{\tilde{\Sigma}}

\newcommand{\avec}{\vec{a}}
\newcommand{\gvec}{\vec{g}}
\newcommand{\hvec}{\vec{h}}

\newcommand{\zvec}{\vec{z}}

\newcommand{\LS}{\operatorname{LS}}
\newcommand{\ML}{\operatorname{ML}}
\newcommand{\RDPG}{\operatorname{RDPG}}

\newcommand{\Bernoulli}{\operatorname{Bernoulli}}
\newcommand{\supp}{\operatorname{supp}}

\newcommand{\onevec}{\mathbf{\vec{1}}}
\newcommand{\UA}{U_A}
\newcommand{\UP}{U_P}
\newcommand{\SA}{S_A}
\newcommand{\SP}{S_P}
\newcommand{\UL}{U_L}

\newcommand{\rls}{r_{\LS}}
\newcommand{\rhatls}{\rhat_{\LS}}
\newcommand{\wls}{w_{\LS}}
\newcommand{\whatls}{\what_{\LS}}

\newcommand{\whatml}{\what_{\ML}}
\newcommand{\thetals}{\theta_{\LS}}
\newcommand{\nuls}{\nu_{\LS}}
\newcommand{\cfw}{c_{F,\wtrue}}
\newcommand{\lambdamin}{\lambda_{\min}}

\newcommand{\inlaw}{\xrightarrow{\calL}}
\newcommand{\inprob}{\xrightarrow{P}}